\documentclass{article}
\usepackage{microtype}
\usepackage{graphicx}
\usepackage{caption}
\usepackage{subcaption}
\usepackage{booktabs} % for professional tables
\usepackage{amsfonts,amssymb,amsmath,exscale,bbm}
\usepackage{bbm}
\usepackage{siunitx} % Required for alignment
\usepackage[export]{adjustbox}
\usepackage[accepted]{icml2018}

\usepackage{url}

\usepackage{amsthm}
\usepackage{mathtools}
\usepackage{appendix}
\usepackage{bm}
\usepackage{bbm}
\usepackage{color}
\usepackage{rotating}
\usepackage{tikz}
\usepackage{mdwtab,booktabs}
\usepackage{makecell}
\usepackage{verbatim}
\usepackage{hyperref}

\icmltitlerunning{Mutual Information Neural Estimation}

\newcommand{\beq}{\begin{equation}}
\newcommand{\eeq}{\end{equation}}
\newcommand{\be}{\begin{equation}}
\newcommand{\ee}{\end{equation}}
\newcommand{\beqa}{\begin{eqnarray}}
\newcommand{\eeqa}{\end{eqnarray}}
\newcommand{\bean}{\begin{eqnarray*}}
\newcommand{\eean}{\end{eqnarray*}}

\newcommand\EE{\mathbb{E}}

\newcommand\NN{\mathbb{N}}

\newcommand\PP{\mathbb{P}}
\newcommand\QQ{\mathbb{Q}}

\newcommand\RR{\mathbb{R}}

\newcommand{\norm}[1]{\left\lVert#1\right\rVert} % Norm

\newcommand{\cF}{\mathcal{F}}
\newcommand\unit{\mathbbm{1}}
\newcommand\Eq{Eqn.}
\newcommand\Eqs{Eqns.}
\newcommand\Fig{Fig.}
\newcommand\Figs{Figs.}

\newcommand\Tbl{Tbl.}

\newcommand{\FGen}{T}
\newcommand{\SN}{\FGen_{\theta}}

\newcommand{\FDG}{f}

\newcommand{\KL}[2]{D_{KL}(#1\mid\mid#2)}

\newcommand{\FD}[2]{D_{\FDG}(#1\mid\mid#2)}

\newcommand{\PJ}[2]{\mathbb{P}_{#1#2}}  % Joint probabilty
\newcommand{\PI}[2]{\mathbb{P}_{#1}\otimes\mathbb{P}_{#2}}  % Product of the marginals
\newcommand\indep{\protect\mathpalette{\protect\indep}{\perp}}
\def\indep#1#2{\mathrel{\rlap{$#1#2$}\mkern2mu{#1#2}}}

\newtheorem{theo}{Theorem}

\newtheorem{lemma}{Lemma}
\theoremstyle{definition}

\theoremstyle{definition}
\newtheorem{definition}{Definition}[section]

\newcommand\R{\mathbb{R}}

\newcommand\PX{\mathbb{P}_X}
\newcommand\PZ{\mathbb{P}_Z}

\newcommand{\ab}[2]{\textcolor{green}{#1 {\em (Aristide: #2)}}}

\DeclareMathOperator*{\argmax}{arg\,max}

\graphicspath{{./figures/}}
\begin{document}

\twocolumn[
\icmltitle{Mutual Information Neural Estimation}

% It is OKAY to include author information, even for blind
% submissions: the style file will automatically remove it for you
% unless you've provided the [accepted] option to the icml2018
% package.

% List of affiliations: The first argument should be a (short)
% identifier you will use later to specify author affiliations
% Academic affiliations should list Department, University, City, Region, Country
% Industry affiliations should list Company, City, Region, Country

% You can specify symbols, otherwise they are numbered in order.
% Ideally, you should not use this facility. Affiliations will be numbered
% in order of appearance and this is the preferred way.
\icmlsetsymbol{equal}{*}

\begin{icmlauthorlist}
\icmlauthor{Mohamed Ishmael Belghazi}{mila}
\icmlauthor{Aristide Baratin}{mila,mcgill}
\icmlauthor{Sai Rajeswar}{mila}
\icmlauthor{Sherjil Ozair}{mila}
\icmlauthor{Yoshua Bengio}{mila,cifar,ivado}
\icmlauthor{Aaron Courville}{mila,cifar}
\icmlauthor{R Devon Hjelm}{mila,ivado}
\end{icmlauthorlist}

\icmlaffiliation{mila}{Montr\'eal Institute for Learning Algorithms (MILA), University of Montr\'eal}
\icmlaffiliation{ivado}{The Institute for Data Valorization (IVADO)}
\icmlaffiliation{cifar}{Canadian Institute for Advanced Research (CIFAR)}
\icmlaffiliation{mcgill}{Department of Mathematics and Statistics, McGill University}
%\icmlaffiliation{goo}{Googol ShallowMind, New London, Michigan, USA}
%\icmlaffiliation{ed}{School of Computation, University of Edenborrow, Edenborrow, United Kingdom}

\icmlcorrespondingauthor{Mohamed Ishmael Belghazi}{ishmael.belghazi@gmail.com}

%\icmlcorrespondingauthor{Eee Pppp}{ep@eden.co.uk}

% You may provide any keywords that you
% find helpful for describing your paper; these are used to populate
% the "keywords" metadata in the PDF but will not be shown in the document
\icmlkeywords{Machine Learning, ICML, Deep Learning, 
			  Neural Networks, GAN, Generative Adversarial Networks, Information Theory}

\vskip 0.3in
]

% this must go after the closing bracket ] following \twocolumn[ ...

% This command actually creates the footnote in the first column
% listing the affiliations and the copyright notice.
% The command takes one argument, which is text to display at the start of the footnote.
% The \icmlEqualContribution command is standard text for equal contribution.
% Remove it (just {}) if you do not need this facility.

\printAffiliationsAndNotice{}  % leave blank if no need to mention equal contribution
%\printAffiliationsAndNotice{\icmlEqualContribution} % otherwise use the standard text.

\begin{abstract}
 We argue that the estimation of mutual information between high dimensional continuous random variables can be achieved by gradient descent over neural networks.
  We present a Mutual Information Neural Estimator (MINE) that is linearly scalable in dimensionality as well as in sample size, trainable through back-prop, and strongly consistent.
We present a handful of applications on which MINE can be used to minimize or maximize mutual information. We apply MINE to improve adversarially trained generative models. We also use MINE to implement the Information Bottleneck, applying it to supervised classification; our results demonstrate substantial improvement in flexibility and performance in these settings.
\end{abstract}

\section{Introduction}
Mutual information is a fundamental quantity for measuring the relationship between random variables. In data science it
has found applications in a wide range of domains and tasks, including biomedical sciences \citep{maes1997multimodality}, blind source separation~\citep[BSS, e.g., independent component analysis,][]{hyvarinen2004independent}, information bottleneck~\citep[IB,][]{tishby2000information}, feature selection~\citep{kwak2002input, peng2005feature}, and causality~\citep{butte2000mutual}.

Put simply, mutual information quantifies the dependence of two random variables $X$ and $Z$. It has the form, 
\begin{align}
I(X; Z) = \int_{\mathcal{X} \times \mathcal{Z}} \log{\frac{d\PJ{X}{Z}}{d\PI{X}{Z}}} d\PJ{X}{Z},
%= \KL{\PJ{X}{Z}}{\PI{X}{Z}}
\label{eq:mutual_information}
\end{align}
where $\PJ{X}{Z}$ is the joint probability distribution, and $\PX =
\int_{\mathcal{Z}} d\PJ{X}{Z}$ and $\PZ = \int_{\mathcal{X}} d\PJ{X}{Z}$ are
the marginals. In contrast to correlation, mutual information captures non-linear statistical dependencies between variables, and thus can act as a measure of true dependence~\citep{kinney2014equitability}.

\begin{comment}
In contrast to correlation, mutual information captures non-linear statistical dependencies between variables, and thus can act as a measure of true dependence.
 Put simply, mutual information is the shared information of two random variables, $X$ and $Z$, defined on the same probability space, $(\mathcal{X} \times \mathcal{Z}, \mathcal{F})$, where $\mathcal{X} \times \mathcal{Z}$ is the domain over both variables (such as $\mathbb{R}^m \times \mathbb{R}^n$), and $\mathcal{F}$ is the set of all possible outcomes over both variables. The mutual information has the form\footnote{We assume the convention that $\log$ is the natural log, so that our units of information are in \emph{nats}.}:
\begin{align}
I(X; Z) = \int_{\mathcal{X} \times \mathcal{Z}} \log{\frac{d\PJ{X}{Z}}{d\PI{X}{Z}}} d\PJ{X}{Z} 
%= \KL{\PJ{X}{Z}}{\PI{X}{Z}}
\label{eq:mutual_information}
\end{align}
where $\PJ{X}{Z} : \mathcal{F} \rightarrow [0, 1]$ is a probabilistic measure (commonly
known as a joint probability distribution in this context), and $\PX =
\int_{\mathcal{Z}} d\PJ{X}{Z}$ and $\PZ = \int_{ \mathcal{X}} d\PJ{X}{Z}$ are
the marginals. 
\end{comment} 
 
  Despite being a pivotal quantity across data science, mutual information has historically been difficult to compute~\citep{paninski2003estimation}.
  Exact computation is only tractable for discrete variables (as the sum can be computed exactly), or for a limited family of problems where the probability distributions are known.
  For more general problems, this is not possible.
  Common approaches are non-parametric~\citep[e.g., binning, likelihood-ratio estimators based on support
  vector machines, non-parametric kernel-density estimators; see,][]{fraser1986independent, darbellay1999estimation, suzuki2008approximating, kwak2002input, moon1995estimation, kraskov2004estimating}, or rely on approximate gaussianity of data distribution~\citep[e.g., Edgeworth expansion,][]{van2005edgeworth}.
  Unfortunately, these estimators typically do not scale well with sample size or dimension~\citep{Gao2014}, and thus cannot be said to be general-purpose.
Other recent works include \citet{Kandamay2017MI, Singh2016MI, Moon2017MI}. 
  
In order to achieve a general-purpose estimator, we rely on the well-known characterization of the mutual information as the Kullback-Leibler (KL-) divergence~\citep{kullback1997information} between the joint distribution and the product of the marginals (i.e., $I(X; Z) = \KL{\PJ{X}{Z}}{\PI{X}{Z}}$).
Recent work uses a dual formulation to cast the estimation of $\FDG$-divergences~\citep[including the KL-divergence, see][]{nguyen2010estimating} as part of an adversarial game between competing deep neural networks~\citep{nowozin2016f}. This approach is at the cornerstone of generative adversarial networks~\citep[GANs, ][]{goodfellow2014generative}, which train a generative model without any explicit assumptions about the underlying distribution of the data.
  
In this paper we demonstrate that exploiting dual optimization to estimate divergences goes beyond the minimax objective as formalized in GANs.
We leverage this strategy to offer a general-purpose parametric neural estimator of mutual information based on dual representations of the KL-divergence~\citep{ruderman2012tighter}, which we show is valuable in settings that do not necessarily involve an adversarial game.
Our estimator is scalable, flexible, and completely trainable via back-propagation.
The contributions of this paper are as follows:
  \begin{itemize}
  \item We introduce the Mutual Information Neural Estimator (MINE), which is scalable, flexible, and completely trainable via back-prop, as well as provide a thorough theoretical analysis.
  \item We show that the utility of this estimator transcends the minimax objective as formalized in GANs, such that it can be used in mutual information estimation, maximization, and minimization.
  \item We apply MINE to palliate mode-dropping in GANs and to improve reconstructions and inference in Adversarially Learned Inference~\citep[ALI,][]{dumoulin2016adversarially} on large scale datasets. 
    %The middle two bullets seem to be about the same application. we should either merge them or better describe the difference.
  \item We use MINE to apply the Information
    Bottleneck method~\cite{tishby2000information} in a continuous setting, and show that this approach outperforms variational bottleneck methods~\citep{Alemi2016deep}.
\end{itemize}

\section{Background}
\subsection{Mutual Information}
  Mutual information is a Shannon entropy-based measure of dependence between random variables. 
  The mutual information between $X$ and $Z$ can be understood as the decrease of the uncertainty in $X$ given $Z$:
 \beq
    I(X;Z) := H(X) - H(X \mid Z),
\eeq
  where $H$ is the Shannon entropy, and $H(X\, |\, Z)$ is the conditional entropy of $Z$ given $X$.
As stated in \Eq~\ref{eq:mutual_information} and the discussion above, the mutual information is equivalent to the Kullback-Leibler (KL-) divergence between the joint, $\PJ{X}{Z}$, and the product of the marginals $\PI{X}{Z}$:
\beq \label{MIdiv}
I(X,Z) = \KL{\PJ{X}{Z}}{\PI{X}{Z}},
\eeq
where $D_{KL}$ is defined as\footnote{Although the discussion is more general, we can think of $\PP$ and $\QQ$ as being distributions on some compact domain $\Omega \subset \RR^d$, with density $p$ and $q$ respect the Lebesgue measure $\lambda$, so that $D_{KL} = \int p \log \frac{p}{q} d\lambda$.},
\begin{align} \label{KL_div}
      \KL{\PP}{\QQ} := 
      %\int_{\Omega} \log\left(\frac{d\PP}{d\QQ}\right) d\PP 
      \EE_{\PP}\left[ \log\frac{d\PP}{d\QQ}\right].
 \end{align}
whenever $\PP$ is absolutely continuous with respect to $\QQ$\footnote{and infinity otherwise. %Note that, for non-singular distributions, the %joint is absolutely continuous with respect to %the product of marginals.
}. 
 
The intuitive meaning of \Eq~\ref{MIdiv} is clear: the larger the divergence between the joint and the product of the marginals, the stronger the dependence between $X$ and $Z$.  This divergence, hence the mutual information, vanishes for fully independent variables.

\subsection{Dual representations of the KL-divergence.}
%\subsubsection{Definition}
A key technical ingredient of MINE are {\it dual representations} of the KL-divergence. 
We will primarily work with the Donsker-Varadhan representation~\citep{DonskerVaradhan}, which results in a tighter estimator; but will also consider the dual $f$-divergence representation~\citep{keziou2003fdivergence,nguyen2010estimating, nowozin2016f}.
% which 
%provides a tight lower-bound on the mutual information.

\begin{comment}
  The KL-divergence between two probability distributions $\PP$ and $\QQ$ on a measure space $\Omega$, with $\PP$ absolutely continuous with respect to $\QQ$ (i.e., $\PP \ll \QQ$), is defined as  
\begin{align} \label{KL_div}
      \KL{\PP}{\QQ} := \int_{\Omega} \log\left(\frac{d\PP}{d\QQ}\right) d\PP = \EE_{\PP}\left[ \log\frac{d\PP}{d\QQ}\right],
 \end{align}
where the argument of the log is the density ratio\footnote{Although the discussion is more general, we can think of $\PP$ and $\QQ$ as being distributions on some compact domain $\Omega \subset \RR^d$, with density $p$ and $q$ respect the Lebesgue measure $\lambda$, so that $D_{KL} = \int p \log \frac{p}{q} d\lambda$.} and $\EE_{\PP}$ denotes the expectation with respect to $\PP$.  It follows from Jensen's inequality that the KL-divergence is always non-negative
and vanishes if and only if $\PP = \QQ$. 
\end{comment}
 
  \paragraph{The Donsker-Varadhan representation.}
  
The following theorem gives a representation of the KL-divergence ~\citep{DonskerVaradhan}:
  \begin{theo}[Donsker-Varadhan representation] \label{DVtheorem}
The KL divergence admits the following dual representation: %ruderman2012tighter}: 
\begin{align}
    \label{eq:donsker}
    \KL{\PP}{\QQ} = \sup_{\FGen : \Omega \to \RR} \EE_{\PP}[\FGen] - \log(\EE_{\QQ}[e^{\FGen}]),
  \end{align}
where the supremum is taken over all functions $T$  such that the two expectations are finite.  \end{theo}
%Then the Donsker-Varadhan representation is given by~\citep{ruderman2012tighter}:
%  \begin{align}c
%    \label{eq:donsker}
%    \KL{\PP}{\QQ} = \sup_{\FGen \in \mathbf{\FGen}} \EE^{\PP}[\FGen] - \log(\EE^{\QQ}%[e^{\FGen}])
%  \end{align}
%  \end{definition}
\proof{See the Supplementary Material.}

\begin{comment}
\ab{}{move proof to Supplementary material}
\begin{proof} A simple proof 
goes as follows. For a given function $\FGen$, consider the Gibbs distribution $\mathbb{G}$ defined by $d \mathbb{G} =  \frac{1}{Z} e^\FGen  d\QQ$, where $Z = \EE_{\QQ}[e^{\FGen}]$. By construction, 
\beq \label{Gibbsequ} 
\EE_{\PP}[\FGen] - \log Z = \EE_{\PP} \left[\log \frac{d\mathbb{G}}{d\QQ}\right]
\eeq
Let $\Delta$ be the gap between the two sides of Equ \ref{eq:donsker}:
\beq \Delta:= \KL{\PP}{\QQ} - \left(\EE_{\PP}[\FGen] - \log(\EE_{\QQ}[e^{\FGen}])\right)
 \eeq
Using  Eqn \ref{Gibbsequ}, we can write $\Delta$ as a KL-divergence: 
\begin{align} \label{gap}
\Delta &= \EE_{\PP}\left[ \log\frac{d\PP}{d\QQ}  - \log \frac{d\mathbb{G}}{d\QQ}\right]  = \EE_{\PP} \log\frac{d\PP}{d\mathbb{G}} \nonumber \\ 
&=  \KL{\PP}{\mathbb{G}} \geq 0
 \end{align} 
The positivity of the KL-divergence gives  $\Delta \geq 0$.  We have thus shown that for any $T$, 
\beq 
\KL{\PP}{\QQ}  \geq \EE_{\PP}[\FGen] - \log(\EE_{\QQ}[e^{\FGen}])
\eeq
and the inequality is preserved upon taking  the supremum over the right-hand side.
Finally, the identity (\ref{gap}) also shows that this bound is {\it tight} whenever $\mathbb{G} = \PP$, namely for an optimal function $T^\ast$ that takes the form 
\beq
T^\ast =  \log \frac{d\PP}{d\QQ} + C
\eeq 
for some constant $C \in \R$. 
\end{proof}
\end{comment} 

A straightforward consequence of Theorem \ref{DVtheorem} is as follows. Let $\cF$ be {\it any} class of functions $\FGen : \Omega \to \RR$ satisfying the integrability constraints of the theorem.
We then have the lower-bound\footnote{The bound in \Eq~\ref{eq:donskerbound} is known as the {\it compression lemma}  in the PAC-Bayes literature ~\citep{Banerjee}.}: 
\beq
\label{eq:donskerbound}
\KL{\PP}{\QQ}  \geq \sup_{\FGen \in \mathcal{F}} \EE_{\PP}[\FGen] - \log(\EE_{\QQ}[e^{\FGen}]).
\eeq
%It will also be useful to introduce the notion of (Donsker) $\cF$-divergence between two probability distributions: 
%\beq 
%D_{\cF}(\PP\mid\mid\QQ) =  \sup_{T\in\cF} \EE_{\PP}[\FGen] - \log(\EE_{\QQ}[e^{\FGen}])
%\eeq
Note also that the bound is {\it tight} for optimal functions $T^\ast$ that relate the distributions to the \emph{Gibbs density} as,
\beq d\PP = \frac{1}{Z} e^{\FGen^\ast}  d\QQ, \,\, \mbox{where} \,\,  Z = \EE_{\QQ}[e^{\FGen^\ast}].
\eeq 
%i.e taking the form $T^\ast\!\! = \log \frac{d\PP}{d\QQ} \! + \! C$
%for some constant $C \!\in\! \R$. 

\paragraph{The $f$-divergence representation.}

It is worthwhile to compare the Donsker-Varadhan representation to the $f$-divergence representation proposed in \citet{nguyen2010estimating, nowozin2016f}, which leads 
to the following bound:
 \beq \label{eq:nguyenKLbound}
\KL{\PP}{\QQ}  \geq \sup_{\FGen \in \mathcal{F}} \EE_{\PP}[\FGen] - \EE_{\QQ}[e^{\FGen-1}].
\eeq
Although the bounds in \Eqs~\ref{eq:donskerbound} and~\ref{eq:nguyenKLbound} are tight for sufficiently large families $\mathcal{F}$, the Donsker-Varadhan bound is {\it stronger} in the sense that, for any fixed $T$,  the right hand side of \Eq~\ref{eq:donskerbound} is larger\footnote{To see this, just apply the identity  $x \geq e \log x$ with $x = \EE_\QQ[e^T]$.} 
than the right hand side of \Eq~\ref{eq:nguyenKLbound}.  
We refer to the work by ~\citet{ruderman2012tighter} for a derivation of both representations in \Eqs~\ref{eq:donskerbound} and~\ref{eq:nguyenKLbound} from the unifying perspective of Fenchel duality.
In Section \ref{MItheo} we discuss versions of MINE based on these two representations, and numerical comparisons are performed in Section~\ref{Sec:MIestimate}. 

\section{The Mutual Information Neural Estimator}
\label{MItheo} 

In this section we formulate the framework of the Mutual Information Neural Estimator (MINE). 
We define MINE and present a theoretical analysis of its consistency and convergence properties. 

\subsection{Method} 

Using both \Eq~\ref{MIdiv} for the mutual information and the dual representation of the KL-divergence, the idea is to choose $\cF$ to be the family of functions $\FGen_\theta :\mathcal{X} \times \mathcal{Z} \to \R$  parametrized by a deep neural network with parameters $\theta \in \Theta$. We call this network the {\it statistics network}.  
We exploit the bound:
\beq \label{MIbound} 
I(X; Z) \geq I_{\Theta}(X,Z),
%D_{\cF}(\PJ{X}{Z}\mid\mid \PI{X}{Z})
\eeq
where  $I_{\Theta}(X,Z)$  is the {\it neural information measure} defined as 
\beq  \label{Fdiv}
I_{\Theta}(X,Z)  =  \sup_{\theta\in\Theta} \EE_{\PJ{X}{Z}}[\FGen_\theta] - \log(\EE_{\PI{X}{Z}}[e^{\FGen_\theta}]).
\eeq
The expectations in \Eq~\ref{Fdiv} are estimated using empirical samples\footnote{Note that samples $\bar{x} \sim \mathbb{P}_X$ and  $\bar{z} \sim \mathbb{P}_Z$  from the marginals are obtained by simply dropping $x, z$ from samples $(\bar{x}, z)$ and $(x, \bar{z}) \sim \PJ{X}{Z}$.} from $\PJ{X}{Z}$ and  $\PI{X}{Z}$ or by shuffling the samples from the joint distribution along the batch axis. The objective can be maximized by gradient ascent.

It should be noted that \Eq~\ref{Fdiv} actually {\it defines} a new class information measures, 
The expressive power of neural network insures that they can approximate the mutual information with arbitrary accuracy.  
%but they may have interesting properties in their own right. 

%A similar line of thoughts in the context of adversarial divergences in generative modeling has recently been %advocated in~\citet{Arora2017} and \citet{Huang2017}.\mib{}{I feel this could confuse the reviewer.}

In what follows, given a distribution $\PP$, we denote by $\hat \PP^{(n)}$ as the empirical distribution associated to $n$ {\it i.i.d.} samples. 
%We define the  \emph{Mutual Information Neural Estimator} (MINE) as follows:
\begin{definition}[Mutual Information Neural Estimator (MINE)]
\label{DefMINE}
Let  $\cF=\{\FGen_{\theta}\}_{\theta\in \Theta}$ be the set of functions parametrized by a neural network.  MINE is defined as, 
% \begin{align*} \label{MINEdef}
%   \widehat{I(X;Z)}_n = \EE_{\hat \PP^{(n)}_{XZ}}[\SNhat(x, z)] - \log(\EE_{\hat \PP^{(n)}_X \otimes  \hat %\PP^{(n)}_Z}[e^{\SNhat(x, z)}])
%  \end{align*}
%  where  
%  $$ \hat \theta_n = \argsup_{\theta \in \Theta}
%  \EE_{\hat \PP^{(n)}_{XZ}}[\SN(x, z)] - \log(\EE_{ \hat \PP^{(n)}_X \otimes  \hat \PP^{(n)}_Z}[e^{\SN(x, %z)}])
%  $$
\beq \label{donskeremp}
\widehat{I(X;Z)}_n = 
\sup_{\theta \in \Theta} \EE_{\PP^{(n)}_{XZ}}[\FGen_\theta] - \log(\EE_{\PP^{(n)}_{X} \otimes \hat \PP^{(n)}_{Z}}[e^{\FGen_\theta}]).
%D_{\cF}( \hat \PP^{(n)}_{XZ} \mid\mid \hat \PP^{(n)}_{X} \otimes \hat \PP^{(n)}_{Z})
\eeq
%where $D_{\cF}$ is defined as in Eqn. (\ref{Fdiv}). 
 \end{definition}

  \begin{algorithm}[ht]
    \begin{algorithmic}
      \STATE $\theta \gets \text{initialize network parameters}$
      \REPEAT
      \STATE Draw $b$ minibatch samples from the joint distribution:      
      \STATE $(\bm{x}^{(1)}, \bm{z}^{(1)}), \ldots, (\bm{x}^{(b)}, \bm{z}^{(b)}) \sim \PJ{X}{Z}$ 
      \STATE Draw $b$ samples from the $Z$ marginal distribution:
      \STATE $\bar{\bm{z}}^{(1)}, \ldots, \bar{\bm{z}}^{(b)} \sim \PP_{Z}$
      \STATE Evaluate the lower-bound:
      \STATE \hspace{-3mm} {\footnotesize $\mathcal{V}(\theta) \gets \frac{1}{b} \sum_{i=1}^{b}\SN(\bm{x}^{(i)},
      \bm{z}^{(i)}) - \log(\frac{1}{b} \sum_{i=1}^{b} e^{\SN(\bm{x}^{(i)},
      \bar{\bm{z}}^{(i)})})$}
      \STATE Evaluate bias corrected gradients (e.g., moving average):
      \STATE $\widehat{G}(\theta) \gets \widetilde{\nabla}_{\theta}\mathcal{V}(\theta)$
      \STATE Update the statistics network parameters:
      \STATE $\theta \gets \theta + \widehat{G}(\theta)$
      \UNTIL{convergence}
    \end{algorithmic} 
    \caption{MINE \label{alg:mi_donsker_estimation}}
  \end{algorithm}

Details on the implementation of MINE are provided in Algorithm~\ref{alg:mi_donsker_estimation}. 
An analogous definition and algorithm also hold for the $\FDG$-divergence formulation in \Eq~\ref{eq:nguyenKLbound}, which we refer to as MINE-$\FDG$.
%where the objective to be optimized is:
%  \beq \label{fdivemp}
%  \EE_{\PP^{(n)}_{XZ}}[\FGen] - \EE_{\PP^{(n)}_{X} \otimes \hat \PP^{(n)}_{Z}}[e^{\FGen-1}],
  %\frac{1}{n} \sum_{i=1}^{n}\SN(x^{(i)},
   %   z^{(i)}) - \frac{1}{n} \sum_{i=1}^{n} e^{\SN(x^{(i)},
   %   \bar{z}^{(i)})-1}
%  \eeq
Since \Eq~\ref{eq:nguyenKLbound} lower-bounds \Eq~\ref{eq:donskerbound}, it generally leads to a {\it looser} estimator of the mutual information, and numerical comparisons of MINE with MINE-$\FDG$ can be found in Section~\ref{Sec:MIestimate}.  
However, in a mini-batch setting, the SGD gradients of MINE are biased. We address this in the next section.

%As is apparent from the MINE objective, the $\log$ term in the Donsker version of MINE introduces a bias to the empirical estimator $\widehat{I(X;Z)}_n$. Specifically, the second term of Eqn \ref{donskeremp}, introduces a bias when used in the minibatch setting due to the $\log$:
%$$\mathbf{bias}(\log \EE_\PP^{n} e^{T_\theta}) = \EE_{\PP}\left[ \log \EE_\PP^{n} e^{T_\theta} \right] - \log \EE_\PP e^{T_\theta} \le 0$$
%The inequality is obtained by applying Jensen's inequality on the first term, and observing that $\EE_{\PP}\left[\EE_\PP^{n}\left[ \cdot \right]\right] = \EE_{\PP}\left[ \cdot \right]$.
%Despite this bias, as we show Section~\ref{theo_analysis}, the empirical estimator is (strongly) consistent,  i.e converges  (almost surely) to  $I_{\cF}(X,Z)$ as the number of samples goes to infinity. 

\subsection{Correcting the bias from the stochastic gradients}  
%Related to the estimator bias, but more problematic from an optimization perspective, is the issue that 
A naive application of stochastic gradient estimation leads to the gradient estimate:
\beq
\widehat{G}_B = \EE_{B}[\nabla_\theta T_\theta] - 
\frac{\EE_{B}[\nabla_\theta T_\theta \, e^{T_\theta}]}{\EE_B \, [e^{T_{\theta}}]}.
\eeq
where, in the second term, the expectations are over the samples of a minibatch $B$, leads to a biased estimate of the full batch gradient\footnote{From the optimization point of view, the $f$-divergence formulation has the advantage of making the use of SGD with unbiased gradients straightforward.}.    
%can lead to a bias in the gradient estimate.\footnote{From the optimization point of view, the $f$-divergence formulation has the advantage of making the use of SGD with unbiased gradients straightforward.} 
%Specifically, the second term of the naive (biased) MINE gradient estimator, 
%\beq
%\widehat{G}_B = \EE_{B}[\nabla_\theta T_\theta] + 
%\frac{\EE_{B}[\nabla_\theta T_\theta \, e^{T_\theta}]}{\EE_B \, [e^{T_{\theta}}]}.
%\eeq

Fortunately, the bias %in the gradient estimate
can be reduced by replacing the estimate in the denominator by an exponential moving average. For small learning rates, this improved MINE gradient estimator can be made to have arbitrarily small bias.\\ 
We found in our experiments that this improves all-around performance of MINE. 
%The Supplementary Material contains a full proof. 

%\ab{}{Fill this section. Briefly explain the bias issue, why it is potentially annoying (we end up maximizing an upper bound of our objective because of Jensen) and how we address it. Refer to Section \ref{num-comparison} for the experiments} 

 \subsection{Theoretical properties} 
 
 \label{theo_analysis}
 
In this section we analyze the consistency and convergence properties of MINE. All the proofs can be found in the Supplementary Material. 

\begin{comment} 
In a nutshell, we show that:

$(i)$ For a large enough statistics networks, the quantity 
\beq 
I_\cF(X,Z) := D_{\cF}(\PJ{X}{Z}\mid\mid \PI{X}{Z})
\eeq
can approximate $I(X,Z)$  with arbitrary accuracy. 

$(ii)$ For a given statistics network, the empirical estimations
\beq 
\widehat{I(X;Z)}_n = D_{\cF}( \hat \PP^{(n)}_{XZ} \mid\mid \hat \PP^{(n)}_{X} \otimes \hat \PP^{(n)}_{Z})
\eeq 
converge to $I_\cF(X,Z)$ almost surely for the choice of samples. \ab{}{add sample complexity} 
\end{comment} 

\subsubsection{Consistency} 
\label{consistency}

MINE relies on a choice of $(i)$ a statistics network and $(ii)$ $n$ samples from the data distribution $\PP_{XZ}$. 

 \begin{definition}[Strong consistency] \label{def-consistency}
 The estimator $\widehat{I(X;Z)}_n$  is strongly consistent if for all $\epsilon >0$, there exists a positive integer $N$ 
 %\in \mathbb{N}$ %where $\mathbb{N}$ is the set of positive integers, 
 and a choice of statistics network such that:
 \[
 \vspace{-1mm}
\forall  n\geq N, \quad  | I(X, Z)-\widehat{I(X;Z)}_n | \leq \epsilon, \, a.e.
  \]
where the probability is over a set of samples. 
 \end{definition}

In a nutshell, the question of consistency is divided into two problems: an {\it approximation} problem related to the size of the family, $\mathcal{F}$, and an {\it estimation} problem related to the use of empirical measures. 
%In fact, the triangular inequality yields 
%\beq
%| \widehat{I(X;Z)}_n - I(X, Z)| \leq |E_1(\cF)| + |E_2(n, \cF)|
%\eeq
% where $E_1(\cF) = I(X, Z)  - \sup_{\FGen \in \mathcal{F}} \EE_{\PJ{X}{Z}}[\FGen] - \log(\EE_{\PI{X}{Z}}[e^{\FGen}])$ and 
% $E_2(n, \cF) = \widehat{I(X;Z)}_n - \sup_{\FGen \in \mathcal{F}} \EE_{\PJ{X}{Z}}[\FGen] - \log(\EE_{\PI{X}{Z}}[e^{\FGen}])$. 
The first problem is addressed by universal approximation theorems for neural networks~\citep{Hornik1989approxtheorem}. 
For the second problem, classical consistency theorems for extremum estimators apply~\citep{deGeer2006Mestimators} under mild conditions on the parameter space. 

This leads to the two lemmas below. 
%We recall our notation for the $\cF$-information measure 
%\beq  
%I_{\cF}(X, Z)  =  \sup_{T\in\cF} \EE_{\PJ{X}{Z}}[\FGen] - \log(\EE_{\PI{X}{Z}}[e^{\FGen}])
%\eeq
%where the divergence $D_{\cF}$ is given by Eqn (\ref{Fdiv}). 
The first lemma states that the neural information measures $I_{\Theta}(X,Z)$, defined in \Eq~\ref{Fdiv}, can approximate the mutual information with arbitrary accuracy: 
\begin{lemma}[approximation] \label{lemma:approximation}
Let $\epsilon >0$. There exists a neural network parametrizing  functions $T_\theta$ with parameters $\theta$ in some compact domain $\Theta \subset \R^k$, such that 
\[ 
|I(X, Z) - I_\Theta(X,Z)| \leq  \epsilon, \, a.e.
\]
%A fortiori, if  $\mathcal{F}$ is any family of functions having $T_{\hat\theta}$ as one of its elements,
%\beq\label{eps-approx}
%|I(X, Z) - I_\cF(X,Z)| \leq  \eta
%\eeq
\end{lemma}
The second lemma states the almost sure convergence of MINE to a neural information measure as the number of samples goes to infinity: 
\begin{lemma}[estimation] \label{lemma:estimation}
Let $\epsilon >0$. Given a family of neural network functions  $T_{\theta}$  with parameters $\theta$ in some bounded domain $\Theta \subset \R^k$, there exists an $N \in \mathbb{N}$, such that 
\beq  \label{eps-estim}
% \forall  n\geq N, \quad \mathrm{Pr}\left(| \widehat{I(X;Z)}_n - I_\Theta(X,Z) | \leq \epsilon \right) = 1.
 \forall  n\geq N, \quad \mid \widehat{I(X;Z)}_n - I_\Theta(X,Z) \mid \leq \epsilon, \, a.e.
% \,\,\, \mbox{with probability one} 
% AARON: I modified the above for space.
 \eeq
\end{lemma} 

Combining the two lemmas with the triangular inequality, we have,
 \begin{theo} \label{theo:consistency}
MINE is strongly consistent. %, in the sense of Def~\ref{def-consistency}.  
 \end{theo} 
%\proof{The proof directly follows from the two Lemmas and the triangular inequality. A complete proof can be found in the Supplementary Material.} 
\begin{comment} 
\begin{proof} 
Let $ \epsilon > 0$. We apply the two Lemmas to find a a family of neural network function $\mathcal{F}$ and  $N \in \NN$ such that (\ref{eps-estim}) and  (\ref{eps-approx}) hold with $\eta = \epsilon/2$. 
By the triangular inequality, for all $n\geq N$ and with probability one, we have: 
\begin{align}
| I(X, Z)-\widehat{I(X;Z)}_n | \quad \leq \quad &|I(X, Z) - \sup_{T_\theta \in \cF} \hat{I}(T_\theta)| + \nonumber \\ &|\widehat{I(X;Z)}_n - I_\cF(X,Z) | \nonumber \\  \leq \quad &\epsilon \nonumber
\end{align}
which proves consistency. 
 \end{proof}
 \end{comment}

\subsubsection{Sample complexity}
\label{samp_complexity}

In this section we discuss the {\it sample complexity} of our estimator.
Since  the focus here is on the empirical estimation problem,  we assume that the mutual information is well enough approximated by the neural information measure $I_\Theta(X, Z)$.
The theorem below is a refinement of Lemma~\ref{lemma:estimation}: it gives how many samples we need for an empirical estimation of the neural information measure at a given accuracy and with high confidence.   

We make the following assumptions: the functions $T_{\theta}$ are $L$-Lipschitz with respect to the parameters $\theta$, and both $T_{\theta}$ and $e^{T_\theta}$ are $M$-bounded (i.e., $|T_\theta|, e^{T_\theta} \leq M$). The domain $\Theta \subset \R^d$ is bounded, so that $\|\theta\| \leq K$ for some constant $K$.   The theorem below shows a sample complexity of $\widetilde{O}\left(\frac{d \log d}{\epsilon^2}\right)$, where $d$ is the dimension of the parameter space.
\begin{theo} \label{theo:rate}
 Given any values $\epsilon,\delta$  of the desired accuracy and confidence parameters, we have,  
\beq  
\mathrm{Pr}\left(| \widehat{I(X;Z)}_n - I_\Theta(X, Z)| \leq \epsilon \right) \geq 1- \delta,
\eeq
whenever the number $n$ of samples satisfies 
\beq 
n \geq \frac{ 2M^2 (d\log (16K L \sqrt{d} /\epsilon) + 2d M + \log(2/\delta))}{\epsilon^2}.
\eeq 
\end{theo} 

\section{Empirical comparisons}
\label{Sec:MIestimate}

Before diving into applications, we perform some simple empirical evaluation and comparisons of MINE. The objective is to show that MINE is effectively able to estimate mutual information and account for non-linear dependence.
%we compare the two variants of MINE we introduced with a baseline non-parametric estimator;

\subsection{Comparing MINE to non-parametric estimation}
We compare MINE and MINE-$\FDG$ to the $k$-NN-based non-parametric estimator found in \citet{kraskov2004estimating}. 
  In our experiment, we consider multivariate Gaussian random
  variables, $X_a$ and $X_b$, with componentwise correlation, $corr(X^i_a, X^j_b) = \delta_{ij} \, \rho$, where $\rho \in (-1,1)$ and $\delta_ij$ is Kronecker's delta.
  %\{$-0.99, -0.9, -0.7, -0.5, -0.3, -0.1, 0., 0.1, 0.3, 0.5, 0.7, 0.9, 0.99$\}$. 
  As the mutual information is invariant to continuous bijective transformations
  of the considered variables, it is enough to consider standardized
  Gaussians marginals. We also compare MINE (using the Donsker-Varadhan representation in \Eq~\ref{eq:donskerbound}) and MINE-$\FDG$ (based on the $\FDG$-divergence representation in \Eq~\ref{eq:nguyenKLbound}).
 
Our results are presented in \Figs~\ref{fig:omie_mutual_information_test}. We observe that both MINE and Kraskov's estimation are virtually indistinguishable from the ground truth when estimating the mutual information between bivariate Gaussians. MINE shows marked improvement over Krakov's when estimating the mutual information between twenty dimensional random variables. We also remark that MINE provides a tighter estimate of the mutual information than MINE-$\FDG$.
  
  \begin{figure}[th]
    \centering
    \begin{tabular}[c]{cc}
       \begin{subfigure}[b]{0.23\textwidth}
     	\includegraphics[width=\textwidth]{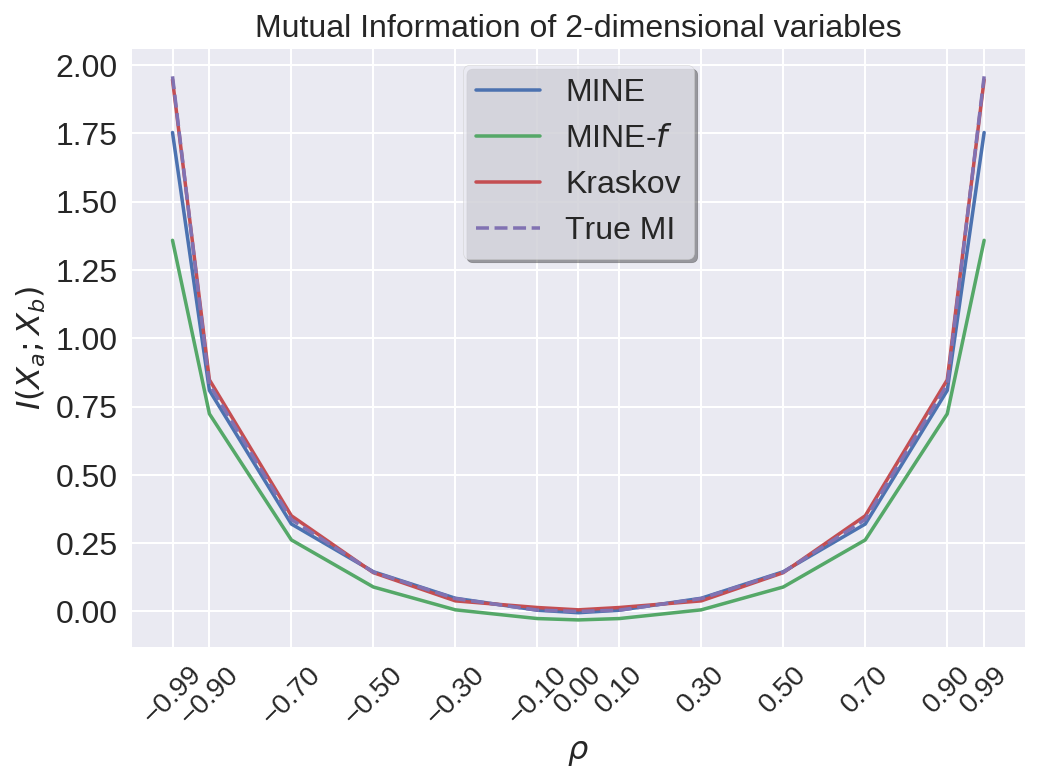}
     \end{subfigure}
     
     \begin{subfigure}[b]{0.23\textwidth}
     	\includegraphics[width=\textwidth]{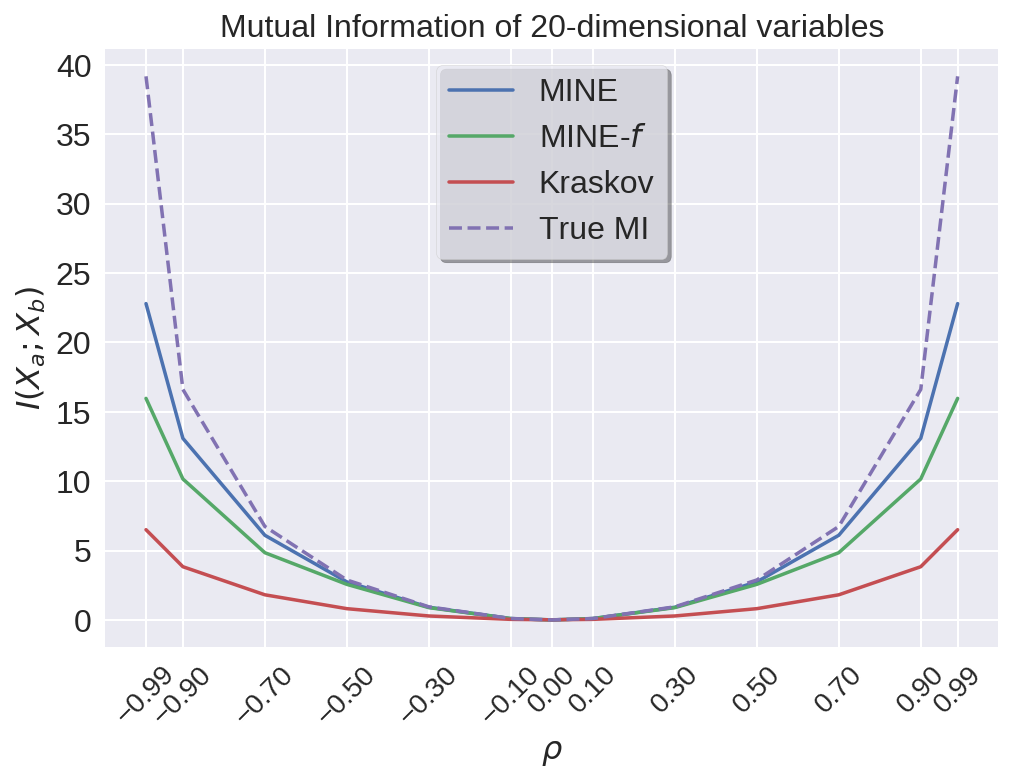}
     \end{subfigure}
     
    \end{tabular}
    \vspace{-1em}
    \caption{Mutual information between two multivariate Gaussians with component-wise correlation $\rho \in (-1,1)$. 
    %[-0.99, -0.9, -0.7, -0.5, -0.3, -0.1, 0., 0.1, 0.3, 0.5, 0.7, 0.9, 0.99]$.
    %MINE estimate is virtually indistinguishable from the ground truth.
    }
    \label{fig:omie_mutual_information_test}
  \end{figure}

\subsection{Capturing non-linear dependencies}
An important property of mutual information between random variables with relationship $Y = f(X) + \sigma \odot \epsilon$, where $f$ is a deterministic non-linear transformation and $\epsilon$ is random noise, is that it is invariant to the deterministic nonlinear transformation, but should only depend on the amount of noise, $\sigma \odot \epsilon$. This important property, that guarantees the quantification dependence without bias for the relationship, is called \emph{equitability}~\citep{kinney2014equitability}.
Our results (\Fig~\ref{fig:nonlinear-mine}) show that MINE captures this important property. 

\begin{figure}[th]
\centering
\includegraphics[width=0.45\textwidth]{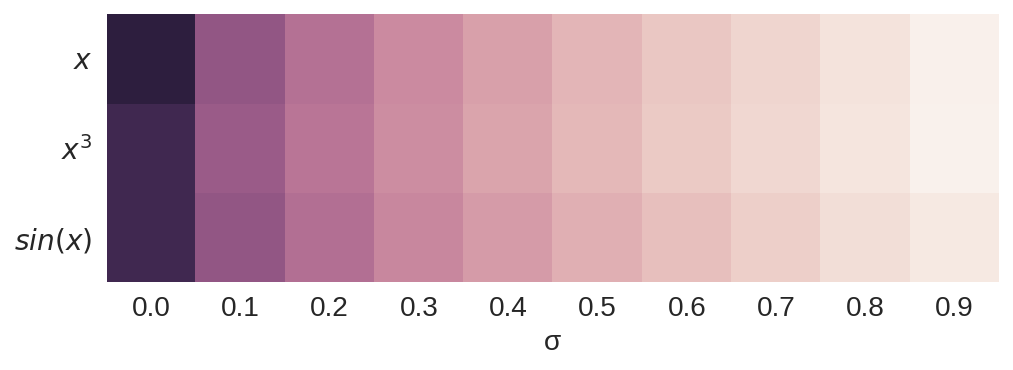}
\vspace{-1.5em}
\caption{MINE is invariant to choice of deterministic nonlinear transformation. The heatmap depicts mutual information estimated by MINE between 2-dimensional random variables $X \sim \mathcal{U}(-1, 1)$ and $Y = f(X) + \sigma \odot \epsilon$, where $f(x) \in \{x, x^3, sin(x)\}$ and $\epsilon \sim \mathcal{N}(0, I)$.}
\vspace{-1.5em}
\label{fig:nonlinear-mine}
  \end{figure}

%   \begin{figure}[ht]
%     \includegraphics[width=\textwidth]{figures/mi_2d_rho_gaussians/gaussian_plot}
%     \caption{\rdh{}{Too much wasted space, not enough clarity} Mutual information between two bivariate gaussians with component-wise correlation of  $corr(X_a, X_b) = \rho \in [-0.99, -0.9, -0.7, -0.5, -0.3, -0.1, 0., 0.1, 0.3, 0.5, 0.7, 0.9, 0.99]$.
%     Note that the MINE is nearly perfect in estimating the mutual information in this setting. \rdh{}{Where are the competing methods? Also the variables in the legend are X and Z. Also I think this plot kinda sucks for telling the story, see notes in text.}}
%     \label{fig:mi_2d_rho_gaussians}
%   \end{figure}
  
 % Next, we compare the estimation of the mutual information in MINE, which relies on the Donsker-Varadhan representation of the variational %lower bound against the one given by ~\citet{nowozin2016f} and used in ~\citet{mescheder2017adversarial}.
%  As we show in Figures~\ref{fig:omie_mutual_information_test} and~\ref{fig:2d_gaussian_kl_vs_donsker}, MINE provides an estimate that much %closer to the ground truth than that provided by ~\citet{nowozin2016f}.

\section{Applications}
\label{Sec:applications}
In this section, we use MINE to present applications of mutual information and compare to competing methods designed to achieve the same goals. 
Specifically, by using MINE to maximize the mutual information, we are able to improve mode representation and reconstruction of generative models.
Finally, by minimizing mutual information, we are able to effectively implement the information bottleneck in a continuous setting.
\subsection{Maximizing mutual information to improve GANs}
\label{sec:Gan1}
Mode collapse~\citep{che2016mode,dumoulin2016adversarially,donahue2016adversarial,Salimans2016gan,Metz2017Unrolled,saatchi2017bayesian,nguyen2017dual,lin2017pacgan,ghosh2017multi} is a common pathology of generative
adversarial networks~\citep[GANs,][]{goodfellow2014generative}, where the
generator fails to produces samples with sufficient diversity (i.e., poorly represent some modes). %Mode collapse can be understood as a greedy behavior of the generator, where the latter simply focuses on a subset of modes which maximize the discriminator's Bayes risk.

GANs as formulated in \citet{goodfellow2014generative} consist of two components: a discriminator, $D:\mathcal{X} \rightarrow [0, 1]$ and a generator, $G: \mathcal{Z} \rightarrow \mathcal{X}$, where $\mathcal{X}$ is a domain such as a compact subspace of $\RR^n$.
Given $Z \in \mathcal{Z}$ follows some simple prior distribution (e.g., a spherical Gaussian with density, $\PP_Z$), the goal of the generator is to match its output distribution to a target distribution, $\PP_X$ (specified by the data samples). The discriminator and generator are optimized through the \emph{value function},
 \begin{align}
 \min_G &\max_D V(D,G) :=
 \nonumber \\
 &\EE_{\PP_X}[D(X)] + \EE_{\PP_Z}[\log{(1 - D(G(Z))}].
 \label{eq:gan_d}
 \end{align}
 
 \begin{comment}
As observed in ~\citet{nowozin2016f}, maximizing the value function amounts to maximizing the variational lower-bound of $2 * D_{JS}(\PP || \QQ) - 2 \log{2}$, where $D_{JS}$ is the Jensen-Shannon divergence.  The generator is then optimized to minimize $V$ alternatively as the discriminator maximizes it. In practice, however, we will use a \emph{proxy} to be maximized by the generator, $\EE_{p_{\tiny \mbox{gen}}}[\log(D(x)]$, which can palliate vanishing gradients.
 \end{comment}
A natural approach to diminish mode collapse would be regularizing the generator's loss with the neg-entropy of the samples. As the sample entropy is intractable, we propose to use the mutual information as a proxy.

%Consider the prior, $Z \sim \mathcal{N}(0,I_{2d})$ as being the concatenation of a code $\zeta \sim \mathcal{N}(0, I_d)$ and noise $\epsilon \sim \mathcal{N}(0,I_d)$. 

Following~\citet{chen2016infogan}, we write the prior as the concatenation of noise and code variables, $Z = [\bm{\epsilon}, \bm{c}]$. We propose to palliate mode collapse by maximizing the mutual information between the samples and the code. $I(G([\bm{\epsilon}, \bm{c}]); \bm{c}) = H(G([\bm{\epsilon}, \bm{c}])) - H(G([\bm{\epsilon}, \bm{c}])\mid\bm{c})$.
%~\footnote{Note that, since $G$ is deterministic and $Z$ and $G(Z)$ are continuous, the mutual information can be infinite. This can be easily fixed e.g by maximizing  the mutual information between $G(Z)$ and a subset of $Z$. However, we did not encounter this problem in our experiments.}.
The generator objective then becomes,
  \begin{align}
    \argmax_G \EE[\log(D(G([\bm{\epsilon}, \bm{c}])))] + \beta I(G([\bm{\epsilon}, \bm{c}]); \bm{c}).
    \label{eq:MINEGAN}
  \end{align}

 As the samples $G([\bm{\epsilon}, \bm{c}])$ are differentiable w.r.t. the parameters of $G$, and the statistics network being a differentiable function, we can maximize the mutual information using back-propagation and gradient ascent by only specifying this additional loss term. Since the mutual information is theoretically unbounded, we use adaptive gradient clipping (see the Supplementary Material) to ensure that the generator receives learning signals similar in magnitude from the discriminator and the statistics network.

\paragraph{Related works on mode-dropping}
Methods to address mode dropping in GANs can readily be found in the literature. 
\citet{Salimans2016gan} use mini-batch discrimination.
In the same spirit, \citet{lin2017pacgan} successfully mitigates mode dropping in GANs by modifying the discriminator to make decisions on multiple real or generated samples.
\citet{ghosh2017multi} uses multiple generators that are encouraged to generate different parts of the target distribution.
\citet{nguyen2017dual} uses two discriminators to minimize the KL and reverse KL divergences between the target and generated distributions.
\citet{che2016mode} learns a reconstruction distribution, then teach the generator to sample from it, the intuition being that the reconstruction distribution is a de-noised or \emph{smoothed} version of the data distribution, and thus easier to learn.
\citet{srivastava2017veegan} minimizes the reconstruction error in the latent space of bi-directional GANs~\citep{dumoulin2016adversarially,donahue2016adversarial}.
\citet{Metz2017Unrolled} includes many steps of the discriminator's optimization as part of the generator's objective. While \citet{chen2016infogan} maximizes the mutual information between the code and the samples, it does so by minimizing a variational upper bound on the conditional entropy~\citep{barber2003algorithm} therefore ignoring the entropy of the samples. \citet{chen2016infogan} makes no claim about mode-dropping. 

\paragraph{Experiments: Spiral, 25-Gaussians datasets}
We apply MINE to improve mode coverage when training a generative adversarial network~\citep[GAN,][]{goodfellow2014generative}.
  We demonstrate using \Eq~\ref{eq:MINEGAN} on the spiral and the 25-Gaussians datasets, comparing two models, one with $\beta = 0$ (which corresponds to the orthodox GAN as in \citet{goodfellow2014generative}) and one with $\beta = 1.0$, which corresponds to mutual information maximization.
   
  % MINEGAN
  \begin{figure}[th]
    \centering
    \begin{tabular}[c]{cc}
      % GAN
      %\begin{subfigure}[b]{0.22\textwidth}
      %	\includegraphics[width=\textwidth,height=30mm]{1000gan.png}
      %   \caption{GAN 1000 iterations }\label{subfig:omiegan_spiral_baseline_1000}
      %\end{subfigure}    
%     \begin{subfigure}[b]{0.22\textwidth}
%     	\includegraphics[width=\textwidth,height=30mm]{3000gan.png}
%        \caption{GAN. 3000 iterations}\label{subfig:omiegan_spiral_baseline_3000}
%     \end{subfigure}
          \begin{subfigure}[l]{0.22\textwidth}
     	\includegraphics[width=\textwidth, height=30mm, clip, trim=1cm .8cm 1cm 1cm]{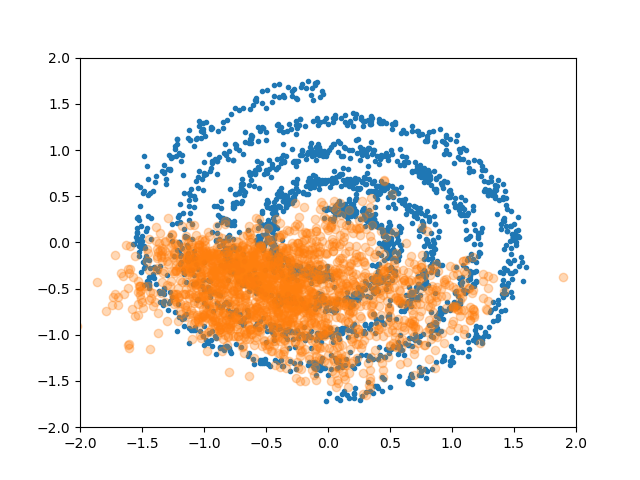}
        \caption{GAN}\label{subfig:omiegan_spiral_baseline_5000}
     \end{subfigure}
       %\begin{subfigure}[b]{0.22\textwidth}      	
       %\includegraphics[width=\textwidth,height=30mm]{1000spiral.png}
       %  \caption{GAN+MINE \\1000 iterations}\label{subfig:omiegan_spiral_1000}
      %\end{subfigure}
%     \begin{subfigure}[b]{0.22\textwidth}
%     	\includegraphics[width=\textwidth,height=30mm]{3000spiral.png}
%        \caption{MINEGAN. 3000 iterations}\label{subfig:omiegan_spiral_3000}
%     \end{subfigure}
     \begin{subfigure}[l]{0.22\textwidth}
     	\includegraphics[width=\textwidth, height=30mm, clip, trim=0 .2cm 0 .2cm]{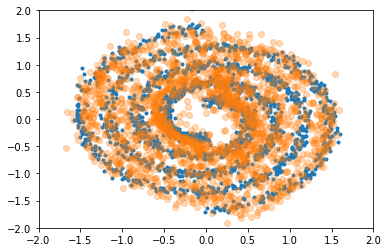}
  \caption{GAN+MINE}\label{subfig:omieomiegan_spiral_5000}
     \end{subfigure} 
    \end{tabular}
    \caption{The generator of the GAN model without mutual information maximization after $5000$ iterations suffers from mode collapse (has poor coverage of the target dataset) compared to GAN+MINE on the spiral experiment.}
    \label{fig:omie_gan_spiral}
  \end{figure}
  
  Our results on the spiral (\Fig~\ref{fig:omie_gan_spiral}) and the $25$-Gaussians (\Fig~\ref{fig:omie_gan_25_gaussians}) experiments both show improved mode coverage over the baseline with no mutual information objective. 
  This confirms our hypothesis that maximizing mutual information helps against mode-dropping in this simple setting.
  
  \begin{figure}[t!]
    \centering
    \begin{tabular}[c]{ccc}
    
     \begin{subfigure}[b]{0.15\textwidth}
     	\includegraphics[width=\textwidth, clip, trim=2.5cm 2cm 2cm 2cm]{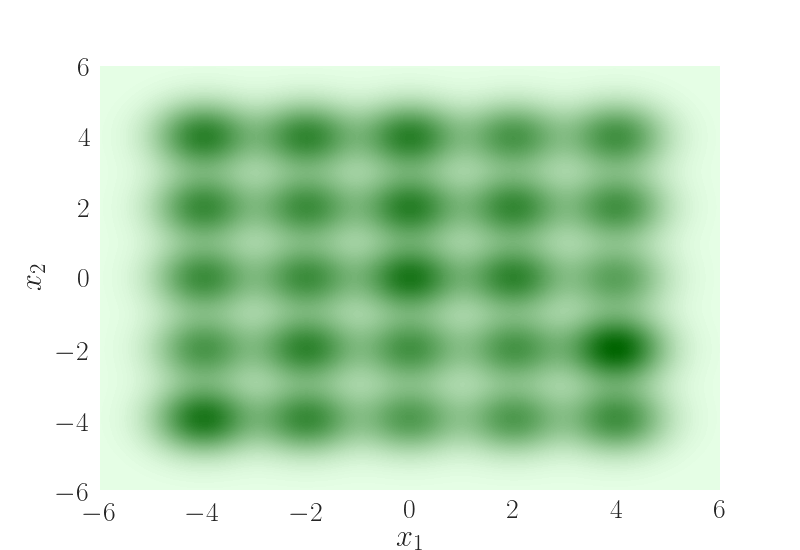}
        \caption{Original data }\label{subfig:omiegan_25gaussian_baseline_0}
     \end{subfigure}
                       
     \begin{subfigure}[b]{0.15\textwidth}
     	\includegraphics[width=\textwidth, clip, trim=2.5cm 2cm 2cm 2cm]{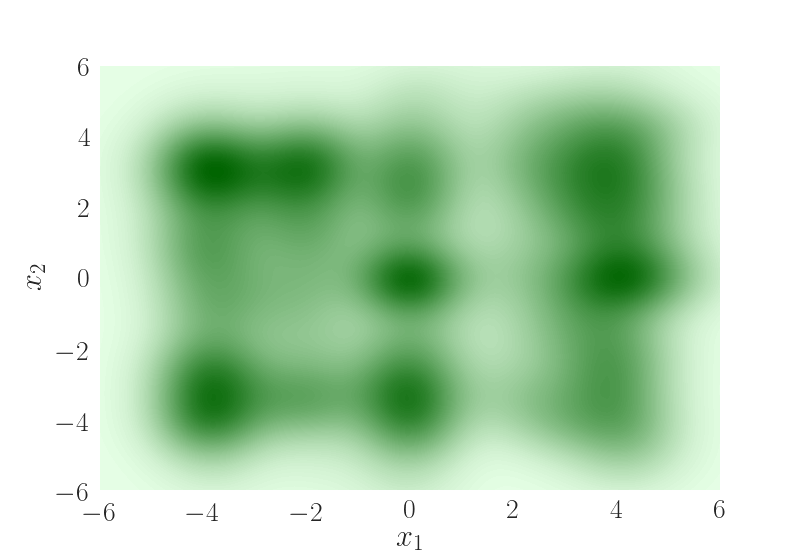}
        \caption{GAN}\label{subfig:omiegan_25gaussian_baseline_20000}
     \end{subfigure}
     
     \begin{subfigure}[b]{0.15\textwidth}
     	\includegraphics[width=\textwidth, clip, trim=2.5cm 2cm 2cm 2cm]{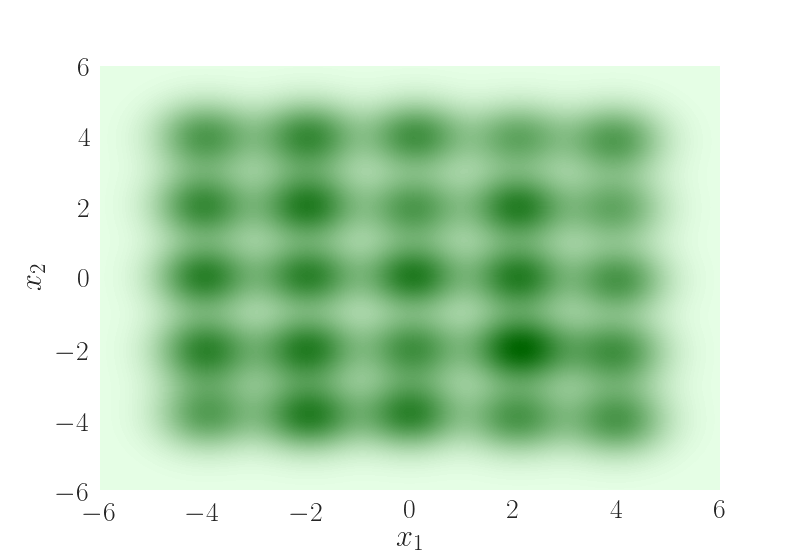}
        \caption{GAN+MINE}\label{subfig:80000_25gaussian_MINE_2}
     \end{subfigure}

    \end{tabular}
    \caption{Kernel density estimate (KDE) plots for GAN+MINE samples and GAN samples on 25 Gaussians dataset.
      }
    \label{fig:omie_gan_25_gaussians}
  \end{figure}
  
\paragraph{Experiment: Stacked MNIST}
Following \citet{che2016mode, Metz2017Unrolled, srivastava2017veegan,lin2017pacgan}, we quantitatively assess MINE's ability to diminish mode dropping on the stacked MNIST dataset which is constructed by stacking three randomly sampled MNIST digits. As a consequence, stacked MNIST offers 1000 modes. Using the same architecture and training protocol as in \citet{srivastava2017veegan,lin2017pacgan}, we train a GAN on the constructed dataset and use a pre-trained classifier on 26,000 samples to count the number of modes in the samples, as well as to compute the KL divergence between the
sample and expected data distributions. Our results in Table~\ref{table:stacked_mnist} demonstrate the effectiveness of  MINE in preventing mode collapse on Stacked MNIST.

\begin{table}[th]
\small
	\begin{center}
  	\begin{tabular}{ lcc}
		& \multicolumn{2}{c}{Stacked MNIST}  \\ \cline{2-3} %&CIFAR-10 \\ \cline{2-3}
    		 &  \makecell{Modes \\(Max 1000)}& KL \\\midrule
    		DCGAN 	& $99.0$ & $3.40$ 		\\%& 0.00844$\pm$0.0020 \\ 
    		ALI   & $16.0$ &  $5.40$ 		\\%& 0.00670$\pm$0.0040 \\
    		Unrolled GAN  & $48.7$ & $4.32$ 		\\%&  0.01300$\pm$0.0009\\
    		VEEGAN  		& $150.0$ &  $2.95$		\\
			PacGAN   	& $1000.0\pm0.0$& $0.06\pm1.0\mathrm{e}^{-2}$\\\midrule
			GAN+MINE (Ours)  	& $1000.0\pm0.0$& $0.05\pm6.9\mathrm{e}^{-3}$  \\\bottomrule
  	\end{tabular}
	\end{center}
	\caption{Number of captured modes and Kullblack-Leibler divergence between the training and samples distributions for DCGAN~\citep{radford2015unsupervised}, ALI~\citep{dumoulin2016adversarially}, Unrolled GAN~\citep{Metz2017Unrolled}, VeeGAN~\citep{srivastava2017veegan}, PacGAN~\citep{lin2017pacgan}.}
	\label{table:stacked_mnist}
\end{table}

  \begin{figure*}[t!]
    \centering
    \begin{tabular}[c]{ccc}
    
     \begin{subfigure}[b]{0.3\textwidth}
     	\includegraphics[width=\textwidth, clip, trim={0 5cm 0 0}]{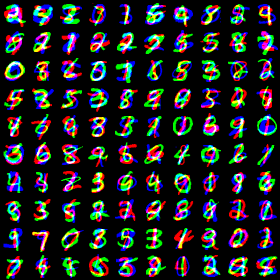}
        \caption{Training set}\label{subfig:stackedmnist_trainingset}
     \end{subfigure}&
                       
     \begin{subfigure}[b]{0.3\textwidth}
     	\includegraphics[width=\textwidth, clip, trim={0 5cm 0 0}]{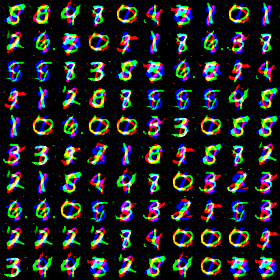}
        \caption{DCGAN}\label{subfig:stackedmnist_dcgan}
     \end{subfigure}&
     
     \begin{subfigure}[b]{0.3\textwidth}
     	\includegraphics[width=\textwidth, clip, trim={0 5cm 0 0}]{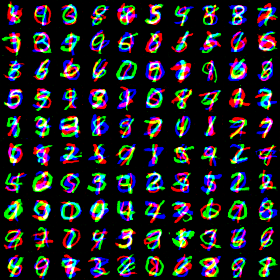}
        \caption{DCGAN+MINE}\label{subfig:stackedmnist_minegan}
     \end{subfigure}
                       
    \end{tabular}
    \caption{Samples from the Stacked MNIST dataset along with generated samples from DCGAN and DCGAN with MINE.
    While DCGAN only shows a very limited number of modes, the inclusion of MINE generates a much better representative set of samples.}
    \label{fig:stacked_mnist}
  \end{figure*}
\subsection{Maximizing mutual information to improve inference in bi-directional adversarial models}
\label{sec:Gan2}
 Adversarial bi-directional models were introduced in Adversarially Learned Inference~\citep[ALI,][]{dumoulin2016adversarially} and BiGAN~\citep{donahue2016adversarial} and are an extension of GANs which incorporate a reverse model, $F: \mathcal{X} \rightarrow \mathcal{Z}$ jointly trained with the generator.
 These models formulate the problem in terms of the value function in \Eq~\ref{eq:gan_d} between two joint distributions, $p(\bm{x}, \bm{z}) = p(\bm{z} \mid \bm{x}) p(\bm{x})$ and $q(\bm{x}, \bm{z}) = q(\bm{x} \mid \bm{z}) p(\bm{z})$ induced by the forward (encoder) and reverse (decoder) models, respectively\footnote{We switch to density notations for convenience throughout this section.}.

One goal of bi-directional models is to do inference as well as to learn a good generative model.
Reconstructions are one desirable property of a model that does both inference and generation, but in practice ALI can lack fidelity
~\citep[i.e., reconstructs less faithfully than desired, see][]{li2017towards,ulyanov2017adversarial,belghazi2018hierarchical}. 
To demonstrate the connection to mutual information, it can be shown (see the Supplementary Material for details) that the reconstruction error, $\mathcal{R}$, is bounded by,
\begin{align} \label{equ:recons_bound}
\mathcal{R} \leq \KL{q(\bm{x}, \bm{z})}{p(\bm{x}, \bm{z})} - I_q(\bm{x}, \bm{z}) + H_q(\bm{z})
\end{align}
If the joint distributions are matched, $H_q(\bm{z})$ tends to $H_p(\bm{z})$, which is fixed as long as the prior, $p(\bm{z})$, is itself fixed. Subsequently, maximizing the mutual information minimizes the expected reconstruction error.

Assuming that the generator is the same as with GANs in the previous section, the objectives for training a bi-directional adversarial model then become:
\begin{align}
\argmax_{D} &\EE_{q(\bm{x}, \bm{z})} [\log{D(\bm{x}, \bm{z})}] + \EE_{p(\bm{x}, \bm{z})}[\log{(1 - D(\bm{x}, \bm{z}))}]
\nonumber \\
\argmax_{F, G} &\EE_{q(\bm{x}, \bm{z})} [\log{(1 - D(\bm{x}, \bm{z}))}] + \EE_{p(\bm{x}, \bm{z})}[\log{D(\bm{x}, \bm{z})}] 
\nonumber \\
&+ \beta I_q(\bm{x}, \bm{z}).
\end{align}

\paragraph{Related works}
\citet{ulyanov2017adversarial} improves reconstructions quality by forgoing the discriminator and expressing the adversarial game between the encoder and decoder. 
\citet{kumar2017improved} augments the bi-directional objective by considering the reconstruction and the corresponding encodings as an additional fake pair.
\citet{belghazi2018hierarchical} shows that a Markovian hierarchical generator in a bi-directional adversarial model provide a hierarchy of reconstructions with increasing levels of fidelity (increasing reconstruction quality).
\citet{li2017towards} shows that the expected reconstruction error can be diminished by minimizing the conditional entropy of the observables given the latent representations. 
The conditional entropy being intractable for general posterior, \citet{li2017towards} proposes to augment the generator's loss with an adversarial cycle consistency loss~\citep{CycleGAN2017} between the observables and their reconstructions.
\vspace{-1cm}
\\

\paragraph{Experiment: ALI+MINE}
  In this section we compare MINE to existing bi-directional adversarial models.  
As the decoder's density is generally intractable, we use three different metrics to measure the fidelity of the reconstructions with respect to the samples; $(i)$ the euclidean reconstruction error, $(ii)$ \emph{reconstruction accuracy}, which is the proportion of labels preserved by the reconstruction as identified by a pre-trained classifier; $(iii)$ the Multi-scale structural similarity metric~\citep[MS-SSIM,][]{MS-SSIM01} between the observables and their reconstructions. 

We train MINE on datasets of increasing order of complexity: a toy dataset composed of 25-Gaussians, MNIST~\citep{lecun1998mnist}, and the CelebA dataset~\citep{liu2015deep}.
\Fig~\ref{fig:bi_directional} shows the reconstruction ability of MINE compared to ALI. Although ALICE does perfect reconstruction (which is in its explicit formulation), we observe significant mode-dropping in the sample space. MINE does a balanced job of reconstructing along with capturing all the modes of the underlying data distribution. 

  Next, we measure the fidelity of the reconstructions over ALI, ALICE, and MINE. 
   \Tbl~2 compares MINE to the existing baselines in terms of
   euclidean reconstruction errors, reconstruction accuracy, and MS-SSIM. On MNIST,
   MINE outperforms ALI in terms of reconstruction errors by a good margin
   and is competitive to ALICE with respect to reconstruction accuracy and
   MS-SSIM. Our results show that MINE's effect on reconstructions
   is even more dramatic when compared to ALI and ALICE on the CelebA dataset.

\begin{figure*}[t]
\title{Reconstructions and Samples}
\begin{minipage}{0.67\textwidth}
  \begin{minipage}{0.24\textwidth}
  \centering
  (a) ALI
  \includegraphics[width=0.99\columnwidth, clip=True, trim=40 20 40 40]{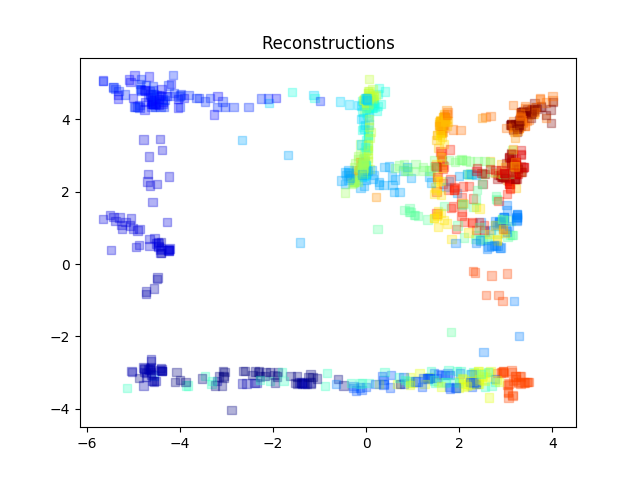}
  \end{minipage}
  \begin{minipage}{0.24\textwidth}
  \centering
  (b) ALICE ($l_2$)
  \includegraphics[width=0.99\columnwidth, clip=True, trim=40 20 40 40]{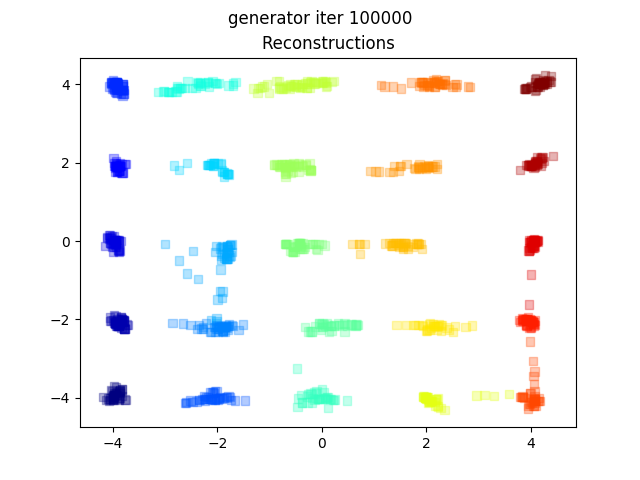}
  \end{minipage}
  \begin{minipage}{0.24\textwidth}
  \centering
  (c) ALICE (A)
  \includegraphics[width=.99\linewidth, clip=True, trim=40 20 40 40]{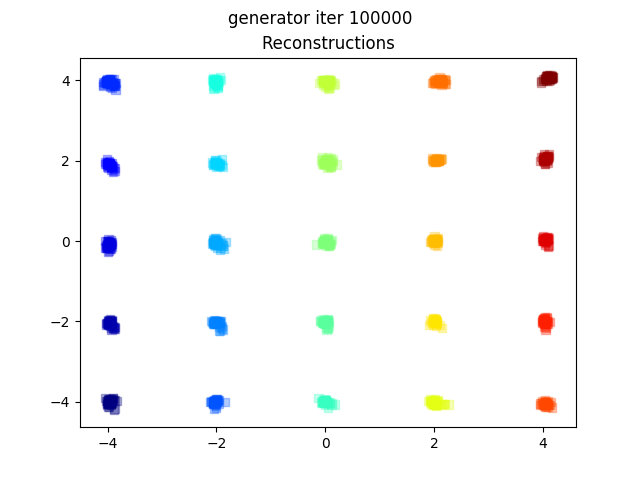}
  \end{minipage}
  \begin{minipage}{0.24\textwidth}
  \centering
  (d) ALI+MINE
  \includegraphics[width=.99\linewidth, clip=True, trim=40 20 40 40]{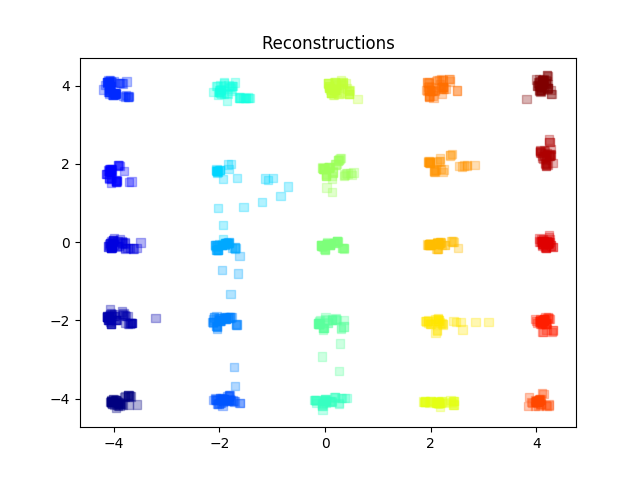}
  \end{minipage}
  \vspace{.3cm}

\begin{minipage}{0.24\textwidth}
  \centering
  \includegraphics[width=0.99\columnwidth, clip=True, trim=40 20 40 40]{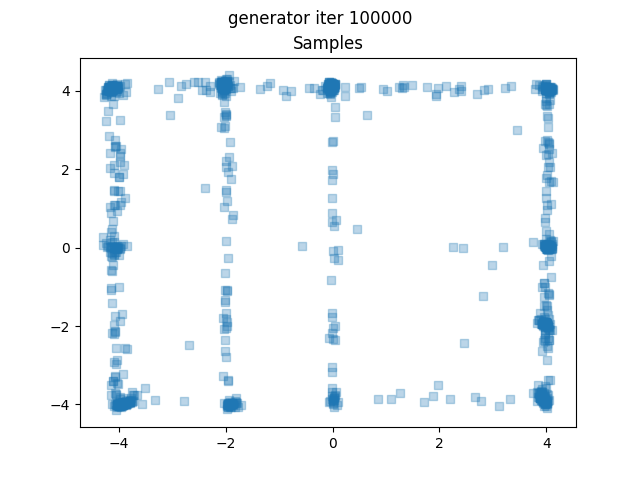}
  \end{minipage}
  \begin{minipage}{0.24\textwidth}
  \centering
  \includegraphics[width=0.99\columnwidth, clip=True, trim=40 20 40 40]{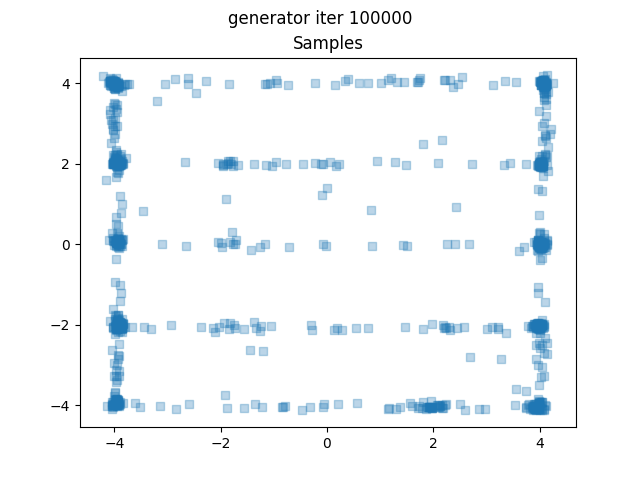}
  \end{minipage}
  \begin{minipage}{0.24\textwidth}
  \centering
  \includegraphics[width=.99\linewidth, clip=True, trim=40 20 40 40]{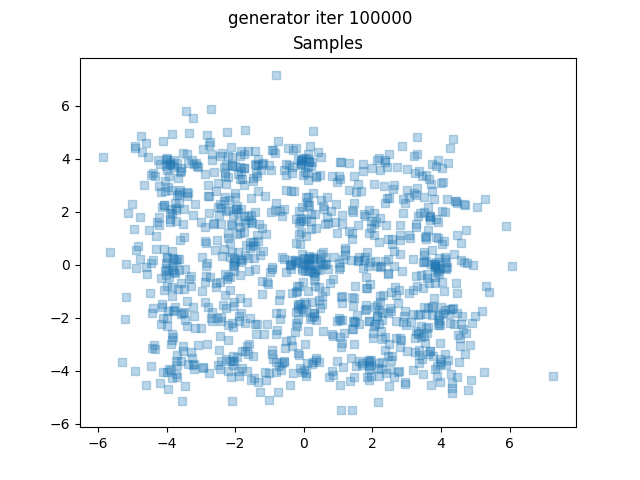}
  \end{minipage}
  \begin{minipage}{0.24\textwidth}
  \centering
  \includegraphics[width=.99\linewidth, clip=True, trim=40 20 40 40]{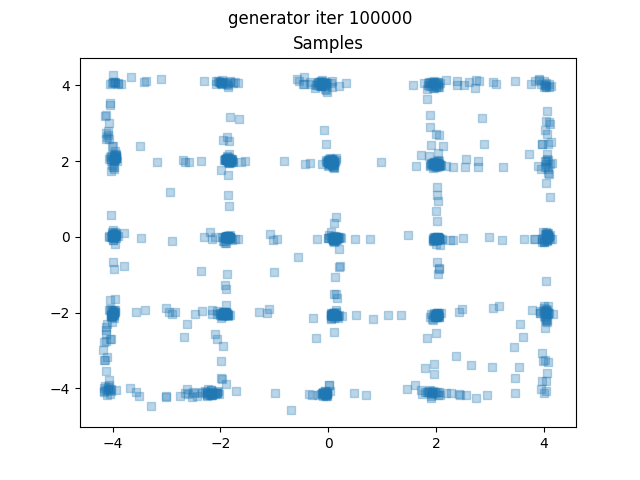}
  \end{minipage}
\end{minipage}
\begin{minipage}{0.31\textwidth}
\caption{Reconstructions and model samples from adversarially learned inference (ALI) and variations intended to increase improve reconstructions.
Shown left to right are the baseline (ALI), ALICE with the $l_2$ loss to minimize the reconstruction error, ALICE with an adversarial loss, and ALI+MINE.
Top to bottom are the reconstructions and samples from the priors.
ALICE with the adversarial loss has the best reconstruction, though at the expense of poor sample quality, where as ALI+MINE captures all the modes of the data in sample space.}
\label{fig:bi_directional}
\end{minipage}
\end{figure*}

\begin{table}[ht]
\small
\centering
\begin{tabular}{l|c|c|c}
\toprule
Model & \makecell{Recons. \\Error}   & \makecell{Recons. \\Acc.(\%)} &  MS-SSIM  \\% \Tstrut\Bstrut\\ 
%Model & Reconstruction. Error  & Reconstruction. Acc.(\%) &  MS-SSIM  \\% \Tstrut\Bstrut\\ 
%\hlx{vhv}
\midrule
\multicolumn{4}{c}{\textbf{MNIST}} \\% \Tstrut\Bstrut\\
\midrule
ALI & 14.24 & 45.95 & 0.97\\
ALICE($l_2$) & 3.20 &99.03 &0.97 \\
ALICE(Adv.) & 5.20 & 98.17 & 0.98\\
MINE & 9.73& 96.10 & 0.99\\
\midrule
\multicolumn{4}{c}{\textbf{CelebA}} \\% \Tstrut\Bstrut\\
\midrule
ALI   & 53.75 & 57.49  & 0.81\\
ALICE($l_2$) & 8.01 & 32.22 &0.93\\
ALICE(Adv.) & 92.56  & 48.95 & 0.51\\
MINE  & 36.11  & 76.08 & 0.99\\
\bottomrule
\end{tabular}
\label{table:mnist}
\caption{Comparison of MINE with other bi-directional adversarial models in terms of euclidean reconstruction error, reconstruction accuracy, and MS-SSIM on the MNIST and CelebA datasets. MINE does a good job compared to ALI in terms of reconstructions. 
Though the explicit reconstruction based baselines (ALICE) can sometimes do better than MINE in terms of reconstructions related tasks, they consistently lag behind in MS-SSIM scores and reconstruction accuracy on CelebA. }
\end{table}

\subsection{Information Bottleneck}
The Information Bottleneck~\citep[IB,][]{tishby2000information} is an information theoretic method for extracting relevant information, or
  yielding a representation, that
  an input $X \in \mathcal{X}$ contains about an output $Y \in \mathcal{Y}$. An optimal representation of $X$ would capture the relevant factors and
  compress $X$ by diminishing the irrelevant parts which do not contribute to
  the prediction of $Y$. IB was recently covered in the context of deep learning~\citep{tishby2015deep}, and as such can be seen as a process to construct an
  approximation of the minimally sufficient statistics of the data.  
  IB seeks an encoder, $q(Z \mid X)$, that induces the
  Markovian structure $X \rightarrow Z \rightarrow Y$. This is done by minimizing the IB Lagrangian,
  \begin{align}
    \mathcal{L}[q(Z \mid X)] = H(Y|Z) + \beta I(X,Z), %\quad
    %\text{or equivalently:} \quad I(X; Y) - \beta I(X; Z).
  \end{align}
which appears as a standard cross-entropy loss augmented with a regularizer promoting minimality of the representation~\citep{AchilleSoatto2017}. Here we propose to estimate the regularizer with MINE. 

%\rdh{}{For me the jump to the second equation isn't so clear. Also isn't one a minimization and the other a %maximization?}
  \paragraph{Related works} In the discrete setting, \cite{tishby2000information} uses
      the Blahut-Arimoto Algorithm \cite{arimoto1972algorithm}, which can be
      understood as cyclical coordinate ascent in function spaces.
      While IB is successful and popular in a discrete
      setting, its application to the continuous setting was stifled by the
      intractability of the continuous mutual information. Nonetheless,
      IB was applied in the case of jointly Gaussian random
      variables in \cite{chechik2005information}.
      
      In order to overcome the intractability of $I(X; Z)$ in the continuous setting, \citet{Alemi2016deep, kolchinsky2017nonlinear, chalk2016relevant} exploit the variational
      bound of \citet{barber2003algorithm} to approximate the conditional entropy in $I(X; Z)$. 
      These approaches differ only on their treatment of the marginal distribution of the bottleneck variable: \citet{Alemi2016deep} assumes a
      standard multivariate normal marginal distribution,
      \citet{chalk2016relevant} uses a Student-t distribution, and \citet{kolchinsky2017nonlinear} uses non-parametric estimators. 
      Due to their reliance on a variational approximation, these methods require a tractable density for the approximate posterior, while MINE does not.
      \vspace{-1cm}
      \\
      \paragraph{Experiment: Permutation-invariant MNIST classification}
      Here, we demonstrate an implementation of the IB objective on
      permutation invariant MNIST using MINE. We compare to the Deep Variational Bottleneck~\citep[DVB,][]{Alemi2016deep} and use the same empirical setup.
      As the DVB relies on a variational bound on the conditional entropy, it therefore requires a tractable density. \citet{Alemi2016deep} opts for a conditional Gaussian encoder $\bm{z} =
      \mu(\bm{x}) + \sigma \odot \epsilon$, where $\epsilon \sim \mathcal{N}(0, I)$. 
      As MINE does not require a tractable density, we consider three type of encoders: $(i)$ a Gaussian encoder as in \citet{Alemi2016deep}; $(ii)$ an \emph{additive noise encoder}, $\bm{z} = enc(\bm{x} + \sigma \odot \epsilon)$; and $(iii)$ a \emph{propagated noise encoder}, $\bm{z} = enc([\bm{x}, \epsilon])$. 
      Our results can be seen in \Tbl~\ref{table:mnist_mi_bottleneck}, and this shows MINE as being superior in these settings.
      
  \begin{table}[h!]
  \label{table:info_bottleneck}
\centering
\begin{tabular}{c|c}

%\hline
\toprule
 Model & Misclass. rate(\%)   \\% \Tstrut\Bstrut\\ 
%\hlx{vhv}
\midrule
Baseline & 1.38\% \\
Dropout & 1.34\% \\
Confidence penalty & 1.36\% \\
Label Smoothing & 1.40\% \\
DVB & 1.13\% \\
DVB + Additive noise & 1.06\% \\
\midrule
MINE(Gaussian) (ours)  & 1.11\% \\
MINE(Propagated) (ours)  & 1.10\% \\
MINE(Additive) (ours) & 1.01\% \\
\bottomrule
%\hline
\end{tabular}
\caption{Permutation Invariant MNIST misclassification rate using ~\citet{Alemi2016deep} experimental setup for regularization by confidence penalty~\citep{pereyra2017regularizing}, label smoothing~\citep{pereyra2017regularizing}, Deep Variational Bottleneck(DVB)~\citep{Alemi2016deep} and MINE.
The misclassification rate is averaged over ten runs. In order to control for the regularizing impact of the additive Gaussian noise in the additive conditional, we also report the results for DVB with additional additive Gaussian noise at the input. All non-MINE results are taken from~\citet{Alemi2016deep}.}
\label{table:mnist_mi_bottleneck}
\end{table}

\section{Conclusion}
We proposed a mutual information estimator, which we called the mutual information neural estimator (MINE), that is scalable in dimension and sample-size. We demonstrated the efficiency of this estimator by applying it in a number of settings. 
First, a term of mutual information  can be introduced alleviate mode-dropping issue in generative adversarial networks~\citep[GANs,][]{goodfellow2014generative}. Mutual information can also be used to improve inference and reconstructions in adversarially-learned inference~\citep[ALI,][]{dumoulin2016adversarially}. Finally, we showed that our estimator allows for tractable application of Information bottleneck methods~\citep{tishby2000information} in a continuous setting.
\section{Acknowledgements}
We would like to thank Martin Arjovsky, Caglar Gulcehre, Marcin Moczulski, Negar Rostamzadeh, Thomas Boquet, Ioannis Mitliagkas, Pedro Oliveira Pinheiro for helpful comments, %and advice, 
as well as Samsung and IVADO for their support. 
%The Institute for Data Valorization (IVADO) for their support. 
%RDH would like to thank The Institute for Data Valorization (IVADO) for their support.

\small
%\bibliography{iclr2017_conference}
%\bibliographystyle{iclr2017_conference}
\bibliography{icml2018}
\bibliographystyle{icml2018}

\newpage
\onecolumn
\section{Appendix}
  
In this Appendix, we provide additional experiment details and spell out the proofs omitted in the text.

\subsection{Experimental Details}

\subsubsection{Adaptive Clipping}
\label{ssec:adaptive_clipping}

Here we assume we are in the context of GANs described in Sections~\ref{sec:Gan1} and~\ref{sec:Gan2}, where the mutual information shows up as a regularizer in the generator objective. 

Notice that the generator is updated by two gradients. The first gradient is
that of the generator's loss, $\mathcal{L}_{g}$ with respect to the
generator's parameters $\theta$, $g_{u} := \frac{\partial
  \mathcal{L}_{g}}{\partial \theta}$. The second flows from the mutual information
  estimate to the generator, $g_{m} := -\frac{\partial \widehat{I(X;Z)}}{\partial\theta}$. If left unchecked, because mutual information is unbounded,
  the latter can overwhelm the former, leading to a failure mode of the algorithm
  where the generator puts all of its attention on maximizing the mutual
  information and ignores the adversarial game with the discriminator. We propose
  to adaptively clip the gradient from the mutual information so that its
  Frobenius norm is at most that of the gradient from the discriminator. Defining $g_a$ to be the adapted gradient following from the statistics network to the generator, we have,
  \begin{align}
    g_a = \min(\norm{g_{u}}, \norm{g_{m}}) \, \frac{g_{m}}{\norm{g_{m}}}.
  \end{align}
Note that adaptive clipping can be considered in any situation where
  MINE is to be maximized.

\subsubsection{GAN+MINE: Spiral and 25-gaussians}
In this section we state the details of experiments supporting mode dropping experiments on the spiral and 25-Gaussians dataset. For both the datasets we use 100,000 examples sampled from the target distributions, using a standard deviation of $0.05$ in the case of 25-gaussians, and using additive noise for the spiral. 
The generator for the GAN consists of two fully connected layers with $500$ units in each layer with batch-normalization~\citep{IoffeS15} and Leaky-ReLU as activation function as in \citet{dumoulin2016adversarially}. 
The discriminator and statistics networks have three fully connected layers with $400$ units each. We use the Adam~\citep{KingmaB14} optimizer with a learning rate of $0.0001$. Both GAN baseline and GAN+MINE were trained for $5,000$  iterations with a mini batch-size of $100$. 

\subsubsection{GAN+MINE: Stacked-MNIST}
Here we describe the experimental setup and architectural details of stacked-MNIST task with GAN+MINE. We compare to the exact same experimental setup followed and reported in PacGAN\cite{lin2017pacgan} and VEEGAN\cite{srivastava2017veegan}. We use a pre-trained classifier to classify generated samples on each of the three stacked channels. Evaluation is done on 26,000 test samples as followed in the baselines.  We train GAN+MINE for 50 epochs on $128,000$ samples. 
Details for generator and discriminator networks are given below in the table\ref{table:stacked_mnist_arc} and table\ref{table:stacked_mnist_arc_dis}. Specifically the statistics network has the same architecture as discriminator in DCGAN with ELU~\citep{ClevertUH15} as activation function for the individual layers and without batch-normalization as highlighted in Table~\ref{table:stacked_mnist_arc_stat}. In order to condition the statistics network on the $z$ variable, we use linear MLPs at each layer, whose output are reshaped to the number of feature maps. The linear MLPs output is then added as a dynamic bias.

\begin{table}[ht]
\centering
\begin{tabular}{l|l|l|l|l}

\hline
\multicolumn{4}{c}{\textbf{Generator}} \\% \Tstrut\Bstrut\\

\hline
  \makecell{Layer}   & \makecell{Number of outputs \\} &  Kernel size &Stride&Activation function\\% \Tstrut\Bstrut\\ 
\hlx{vhv}

 Input $\bm{z} \sim \mathcal{U}(-1, 1)^{100}$ & 100 & && \\
Fully-connected& 2*2*512 &&&ReLU\\
Transposed convolution & 4*4*256 & $5 * 5$&2&ReLU\\
 Transposed convolution & 7*7*128 & $5 * 5$&2&ReLU\\
 Transposed convolution& 14*14*64 & $5 * 5$&2&ReLU\\
Transposed convolution& 28*28*3 & $5 * 5$&2&Tanh\\
\hline
\end{tabular}
\caption{Generator network for Stacked-MNIST experiment using GAN+MINE.}
\label{table:stacked_mnist_arc}
\end{table}

\begin{table}[H]
\centering
\begin{tabular}{l|l|l|l|l}

\hline
\multicolumn{4}{c}{\textbf{Discriminator}} \\% \Tstrut\Bstrut\\

\hline
  \makecell{Layer}   & \makecell{Number of outputs \\} &  Kernel size &Stride&Activation function\\% \Tstrut\Bstrut\\ 
\hlx{vhv}

 Input $x$  & $28 * 28 * 3$ & && \\

Convolution & 14*14*64 & $5 * 5$&2&ReLU\\
 Convolution & 7*7*128 & $5 * 5$&2&ReLU\\
 Convolution& 4*4*256 & $5 * 5$&2&ReLU\\
Convolution& 2*2*512 & $5 * 5$&2&ReLU\\
Fully-connected& 1 & 1& Valid &Sigmoid\\
\hline
\end{tabular}
\caption{Discriminator network for Stacked-MNIST experiment.}
\label{table:stacked_mnist_arc_dis}
\end{table}

\begin{table}[H]
\centering
\begin{tabular}{l|l|l|l|l}

\hline
\multicolumn{4}{c}{\textbf{Statistics Network}} \\% \Tstrut\Bstrut\\

\hline
  \makecell{Layer}   & \makecell{number of outputs \\} &  kernel size &stride&activation function\\% \Tstrut\Bstrut\\ 
\hlx{vhv}

 Input $x, z$  &  & && \\

Convolution & 14*14*16 & $5 * 5$&2&ELU\\
 Convolution & 7*7*32 & $5 * 5$&2&ELU\\
 Convolution & 4*4*64 & $5 * 5$&2&ELU\\
 Flatten &-&-&-&-\\
Fully-Connected& 1024 & 1& Valid& None\\
Fully-Connected& 1 &1 & Valid & None\\
\hline
\end{tabular}
\caption{Statistics network for Stacked-MNIST experiment.}
\label{table:stacked_mnist_arc_stat}
\end{table}

\subsubsection{ALI+MINE: MNIST and CelebA}
In this section we state the details of experimental setup and the network architectures used for the task of improving reconstructions and representations in bidirectional adversarial models with MINE.  The  generator and  discriminator network architectures along with the hyper parameter setup used in these tasks are similar to the ones used in DCGAN~\citep{radford2015unsupervised}. 

Statistics network conditioning on the latent code was done as in the Stacked-MNIST experiments. We used Adam as the optimizer with a learning rate of 0.0001.
We trained the model for a total of $35,000$ iterations on CelebA and $50,000$ iterations on MNIST, both with a mini batch-size of $100$.
\begin{table}[H]
\centering
\begin{tabular}{l|l|l|l|l}
\hline
\multicolumn{4}{c}{
\centering
\textbf{Encoder}} \\% \Tstrut\Bstrut\\

\hline
  \makecell{Layer}   & \makecell{Number of outputs \\} &  Kernel size &Stride&Activation function\\% \Tstrut\Bstrut\\ 
\hlx{vhv}

 Input $[\bm{x}, \bm{\epsilon}]$ & 28*28*129& && \\
Convolution & 14*14*64 & $5 * 5$&2&ReLU\\
 Convolution & 7*7*128 & $5 * 5$&2&ReLU\\
 Convolution& 4*4*256 & $5 * 5$&2&ReLU\\
Convolution & 256& $4 * 4$&Valid&ReLU\\
Fully-connected& 128 &-&-&None\\
\hline
\end{tabular}
\caption{Encoder network for bi-directional models on MNIST.
$\bm{\epsilon} \sim \mathcal{N}_{128}(0, I)$.}

\label{table:mnist_ali_enc}
\end{table}

\begin{table}[ht]
\centering
\begin{tabular}{l|l|l|l|l}

\hline
\multicolumn{4}{c}{\textbf{Decoder}} \\% \Tstrut\Bstrut\\

\hline
  \makecell{Layer}   & \makecell{Number of outputs \\} &  Kernel size &Stride&Activation function\\% \Tstrut\Bstrut\\ 
\hlx{vhv}

 Input $\bm{z}$ & 128 & && \\
Fully-connected& 4*4*256 &&&ReLU\\
 Transposed convolution & 7*7*128 & $5 * 5$&2&ReLU\\
 Transposed convolution& 14*14*64 & $5 * 5$&2&ReLU\\
Transposed convolution& 28*28*1 & $5 * 5$&2&Tanh\\
\hline
\end{tabular}
\caption{Decoder network for bi-directional models on MNIST. $\bm{z} \sim \mathcal{N}_{256}(0, I)$}

\label{table:mnist_ali_dec}
\end{table}

\begin{table}[H]
\centering
\begin{tabular}{l|l|l|l|l}

\hline
\multicolumn{4}{c}{\textbf{Discriminator}} \\% \Tstrut\Bstrut\\

\hline
  \makecell{Layer}   & \makecell{Number of outputs \\} &  Kernel size &Stride&Activation function\\% \Tstrut\Bstrut\\ 
\hlx{vhv}

 Input $x$  & $28 * 28 * 3$ & && \\

 Convolution & 14*14*64 & $5 * 5$&2&LearkyReLU\\
 Convolution & 7*7*128 & $5 * 5$&2&LeakyReLU\\
 Convolution& 4*4*256 & $5 * 5$&2&LeakyReLU\\
 Flatten &-&-&-\\
Concatenate $\bm{z}$ & -&-&-\\
Fully-connected& 1024 & -& -& LeakyReLU\\
Fully-connected& 1 & -& - &Sigmoid\\
\hline
\end{tabular}
\caption{Discriminator network for bi-directional models experiments MINE on MNIST.}
\label{table:mnist_ali_disc}
\end{table}

\begin{table}[H]
\centering
\begin{tabular}{l|l|l|l|l}

\hline
\multicolumn{4}{c}{\textbf{Statistics Network}} \\% \Tstrut\Bstrut\\

\hline
  \makecell{Layer}   & \makecell{number of outputs \\} &  kernel size &stride&activation function\\% \Tstrut\Bstrut\\ 
\hlx{vhv}

 Input $\bm{x}, \bm{z} $ &  & && \\
Convolution & 14*14*64 & $5 * 5$&2&LeakyReLU\\
 Convolution & 7*7*128 & $5 * 5$&2&LeakyReLU\\
 Convolution & 4*4*256 & $5 * 5$&2&LeakyReLU\\
Flatten& - & -&-&-\\
Fully-connected& 1 &-&-&None\\
\hline
\end{tabular}
\caption{Statistics network for bi-directional models using MINE on MNIST.}
\label{table:mnist_ali_stat}
\end{table}

\begin{table}[H]
\centering
\begin{tabular}{l|l|l|l|l}
\hline
\multicolumn{4}{c}{
\centering
\textbf{Encoder}} \\% \Tstrut\Bstrut\\

\hline
  \makecell{Layer}   & \makecell{Number of outputs \\} &  Kernel size &Stride&Activation function\\% \Tstrut\Bstrut\\ 
\hlx{vhv}

 Input $[\bm{x}, \bm{\epsilon}]$ & 64*64*259  & && \\

Convolution & 32*32*64 & $5 * 5$&2&ReLU\\
 Convolution & 16*16*128 & $5 * 5$&2&ReLU\\
 Convolution& 8*8*256 & $5 * 5$&2&ReLU\\
 Convolution& 4*4*512 & $5 * 5$&2&ReLU\\
Convolution & 512 & $4 * 4$&Valid&ReLU\\
Fully-connected& 256 &-&-&None\\
\hline
\end{tabular}
\caption{Encoder network for bi-directional models on CelebA. $\bm{\epsilon} \sim \mathcal{N}_{256}(0, I)$.}

\label{table:celeba_ali_enc}
\end{table}

\begin{table}[ht]
\centering
\begin{tabular}{l|l|l|l|l}

\hline
\multicolumn{4}{c}{\textbf{Decoder}} \\% \Tstrut\Bstrut\\

\hline
  \makecell{Layer}   & \makecell{Number of outputs \\} &  Kernel size &Stride&Activation function\\% \Tstrut\Bstrut\\ 
\hlx{vhv}

 Input $\bm{z} \sim \mathcal{N}_{256}(0, I)$ & 256 & && \\
Fully-Connected& 4*4*512&-&-&ReLU\\
Transposed convolution & 8*8*256 & $5 * 5$&2&ReLU\\
 Transposed convolution & 16*16*128 & $5 * 5$&2&ReLU\\
 Transposed convolution& 32*32*64 & $5 * 5$&2&ReLU\\
Transposed convolution& 64*64*3 & $5 * 5$&2&Tanh\\
\hline
\end{tabular}
\caption{Decoder network for bi-directional model(ALI, ALICE) experiments using MINE on CelebA.}

\label{table:celeba_ali_dec}
\end{table}

\begin{table}[H]
\centering
\begin{tabular}{l|l|l|l|l}

\hline
\multicolumn{4}{c}{\textbf{Discriminator}} \\% \Tstrut\Bstrut\\

\hline
  \makecell{Layer}   & \makecell{Number of outputs \\} &  Kernel size &Stride&Activation function\\% \Tstrut\Bstrut\\ 
\hlx{vhv}

 Input $x$  & $64 * 64 * 3$ & && \\
 Convolution & 32*32*64 & $5 * 5$&2&LearkyReLU\\
 Convolution & 16*16*128 & $5 * 5$&2&LeakyReLU\\
 Convolution& 8*8*256 & $5 * 5$&2&LeakyReLU\\
 Convolution& 4*4*512& $5 * 5$&2&LeakyReLU\\
 Flatten &-&-&-\\
Concatenate $\bm{z}$ & -&-&-\\
Fully-connected& 1024 & -& -& LeakyReLU\\
Fully-connected& 1 & -& - &Sigmoid\\
\hline
\end{tabular}
\caption{Discriminator network for bi-directional models on CelebA.}
\label{table:celeba_ali_disc}
\end{table}

\begin{table}[H]
\centering
\begin{tabular}{l|l|l|l|l}

\hline
\multicolumn{4}{c}{\textbf{Statistics Network}} \\% \Tstrut\Bstrut\\

\hline
  \makecell{Layer}   & \makecell{number of outputs \\} &  kernel size &stride&activation function\\% \Tstrut\Bstrut\\ 
\hlx{vhv}

 Input $\bm{x}, \bm{z} $ &  & && \\
Convolution & 32*32*16 & $5 * 5$&2&ELU\\
 Convolution & 16*16*32 & $5 * 5$&2&ELU\\
 Convolution & 8*8*64 & $5 * 5$&2&ELU\\
 Convolution & 4*4*128 & $5 * 5$&2&ELU\\
Flatten& - & -&-&-\\
Fully-connected& 1 &-&-&None\\
\hline
\end{tabular}
\caption{Statistics network for bi-directional models on CelebA.}
\label{table:celeba_ali_stat}
\end{table}

\subsubsection{Information bottleneck with MINE}
In this section we outline the network details and hyper-parameters used for the information bottleneck task using MINE. To keep comparison fair all hyperparameters and architectures are those outlined in ~\citet{Alemi2016deep}. The statistics network is shown, a two layer MLP with additive noise at each layer and 512 ELUs~\citep{ClevertUH15} activations, is outlined in table\ref{table:mnist_IB}.
 
 \begin{table}[H]
\centering
\begin{tabular}{l|l|l}

\hline
\multicolumn{2}{c}{\textbf{Statistics Network}} \\% \Tstrut\Bstrut\\

\hline
  \makecell{Layer}   & \makecell{number of outputs \\} &activation function\\% \Tstrut\Bstrut\\ 
\hlx{vhv}

 input $[\bm{x}, \bm{z}]$  &  &  \\

Gaussian noise(std=0.3) & - &-\\
 dense layer & 512 & ELU\\
 Gaussian noise(std=0.5) & - &-\\
 dense layer& 512&ELU\\
Gaussian noise(std=0.5) & - &-\\
dense layer& 1 & None\\
\hline
\end{tabular}
\caption{Statistics network for Information-bottleneck experiments on MNIST.}
\label{table:mnist_IB}
\end{table}
 \subsection{Proofs}
 
 \subsubsection{Donsker-Varadhan Representation}
 
  \begin{theo}[Theorem \ref{DVtheorem} restated] 
The KL divergence admits the following dual representation: %ruderman2012tighter}: 
\begin{align}
    \label{eq:donsker2}
    \KL{\PP}{\QQ} = \sup_{\FGen : \Omega \to \RR} \EE_{\PP}[\FGen] - \log(\EE_{\QQ}[e^{\FGen}]),
  \end{align}
where the supremum is taken over all functions $T$  such that the two expectations are finite.  \end{theo}

\begin{proof} A simple proof  goes as follows. For a given function $\FGen$, consider the Gibbs distribution $\mathbb{G}$ defined by $d \mathbb{G} =  \frac{1}{Z} e^\FGen  d\QQ$, where $Z = \EE_{\QQ}[e^{\FGen}]$. By construction, 
\beq \label{Gibbsequ} 
\EE_{\PP}[\FGen] - \log Z = \EE_{\PP} \left[\log \frac{d\mathbb{G}}{d\QQ}\right]
\eeq
Let $\Delta$ be the gap, %between the two sides of \Eq~\ref{eq:donsker}:
\beq \Delta:= \KL{\PP}{\QQ} - \left(\EE_{\PP}[\FGen] - \log(\EE_{\QQ}[e^{\FGen}])\right)
 \eeq
Using  Eqn \ref{Gibbsequ}, we can write $\Delta$ as a KL-divergence: 
\beq \label{gap}
\Delta = \EE_{\PP}\left[ \log\frac{d\PP}{d\QQ}  - \log \frac{d\mathbb{G}}{d\QQ}\right]  = \EE_{\PP} \log\frac{d\PP}{d\mathbb{G}}  
=  \KL{\PP}{\mathbb{G}} 
 \eeq 
The positivity of the KL-divergence gives  $\Delta \geq 0$.  We have thus shown that for any $T$, 
\beq 
\KL{\PP}{\QQ}  \geq \EE_{\PP}[\FGen] - \log(\EE_{\QQ}[e^{\FGen}])
\eeq
and the inequality is preserved upon taking  the supremum over the right-hand side.
Finally, the identity (\ref{gap}) also shows that this bound is {\it tight} whenever $\mathbb{G} = \PP$, namely for optimal functions $T^\ast$  taking the form 
$
T^\ast =  \log \frac{d\PP}{d\QQ} + C
$ 
for some constant $C \in \R$. 
\end{proof}

  \subsubsection{Consistency Proofs}
  \label{sec:proofs}
  
 This section presents the proofs of the Lemma and consistency theorem stated in the consistency in Section \ref{consistency}.

 In what follows, we assume that the input space $\Omega =  \mathcal{X}\times\mathcal{Z}$  is a compact domain of $\R^d$,  and all measures are absolutely continuous with respect to the Lebesgue measure.    %; and we write $\omega =  (x, z)$ for an element of $\Omega$. 
We will restrict to families of feedforward functions with continuous activations, with a single output neuron, so that a given architecture defines a continuous mapping $(\omega, \theta) \to T_\theta(\omega)$ from $\Omega \times \Theta$ to $\RR$.  

To avoid unnecessary heavy notation, we denote $\PP = \PP_{XZ}$ and $\QQ=\PI{X}{Z}$ for the joint distribution and product of marginals, and $\PP_n, \QQ_n$ for their empirical versions.  
We will use the notation $\hat{I}(T)$ for the quantity: 
\beq 
\hat{I}(T) = \EE_{\PP}[\FGen] - \log(\EE_{\QQ}[e^{\FGen}])
\eeq
so that $I_{\Theta}(X, Z) = \sup_{\theta\in\Theta} \hat{I}(T_\theta)$.
 \begin{lemma}[Lemma \ref{lemma:approximation} restated]
Let $\eta >0$. There exists a family of neural  network functions $T_\theta$ with parameters $\theta$ in some compact domain $\Theta \subset \R^k$, such that 
\beq \label{approx}
|I(X, Z) - I_\Theta(X,Z)| \leq  \eta
\eeq
where 
\beq  
I_{\Theta}(X, Z)  =  \sup_{\theta\in \Theta} \EE_{\PJ{X}{Z}}[\FGen_\theta] - \log(\EE_{\PI{X}{Z}}[e^{\FGen_\theta}])
\eeq
%A fortiori, if  $\mathcal{F}$ is any family of functions having $T_{\hat\theta}$ as one of its elements,
%\beq\label{eps-approx}
%|I(X, Z) - I_\cF(X,Z)| \leq  \eta
%\eeq
\end{lemma}

\begin{proof}  
Let  $T^\ast  = \log \frac{d\PP}{d\QQ}$. By construction, $T^\ast$ satisfies:
\beq\EE_\PP [T^\ast] = I(X, Z),  \qquad \EE_\QQ [e^{T^\ast}] = 1\eeq
For a function $T$, the (positive) gap $I(X, Z) - \hat{I}(T)$ can be written as
\beq \label{equ:bound}
I(X, Z) - \hat{I}(T) = \EE_\PP[T^\ast - T] + \log \EE_\QQ[e^{T}] \leq \EE_\PP[T^\ast - T] + \EE_\QQ[e^{T} - e^{T^\ast}]  
\eeq
where we used the inequality $\log x \leq x-1$. 

Fix $\eta >0$. We first consider the case where $T^\ast$ is {\it bounded} from above by a constant $M$. 
By the universal approximation theorem (see corollary 2.2 of~\citet{Hornik1989approxtheorem}\footnote{Specifically, the argument relies on the density of feedforward network functions in the space $L^1(\Omega, \mu)$ of integrable functions with respect the measure $\mu=\PP+\QQ$.}), we may choose a feedforward network function $T_{\hat \theta} \leq M$  
such that    
\beq 
\EE_\PP|T^\ast-\FGen_{\hat \theta} | \leq \frac{\eta}{2} \quad \mbox{and} \quad   \EE_\QQ|T^\ast-\FGen_{\hat \theta}| \leq \frac{\eta}{2} e^{-M}
\eeq
Since $\exp$ is Lipschitz continuous with constant $e^M$ on $(-\infty, M]$,  we have
 \beq
 \EE_\QQ |e^{T^\ast} - e^{\FGen_{\hat \theta}}|  \leq e^M \, \EE_\QQ|T^\ast - \FGen_{\hat \theta} | \leq\frac{\eta}{2} \eeq
From Equ~\ref{equ:bound} and the triangular inequality, we then obtain:
 \beq \label{proof-inequality}
 |I(X, Z) - \hat{I}(T_{\hat\theta})|  \leq \EE_\PP |T^\ast - T_{\hat \theta}| +  \EE_\QQ|e^{T^\ast} - e^{\FGen_{\hat \theta}}|
 \leq \frac{\eta}{2}  + \frac{\eta}{2}   \leq \eta
\eeq
In the general case, the idea is to partition $\Omega$ in two subset $\{T^\ast>M\}$ and $\{T^\ast \leq M\}$ for a suitably chosen large value of $M$. For a given subset $S \subset \Omega$, we will denote by $\unit_{S}$ its characteristic function,  $\unit_S(\omega) = 1$ if $\omega\in S$ and $0$ otherwise. $T^\ast$ is integrable with respect to $\PP$\footnote{This can be seen from the identity~\citep{Gyorfi1987}
$$\EE_{\PP}\left|\log\frac{d \PP}{d\QQ}\right|\leq \KL{\PP}{\QQ} + 4 \sqrt{\KL{\PP}{\QQ}}$$}, and $e^{T^\ast}$ is integrable with respect to $\QQ$, so by 
the dominated convergence theorem, we may choose $M$ so that the expectations $\EE_\PP [T^\ast \unit_{T^\ast>M}]$ and $\EE_\QQ[e^{T^\ast} \unit_{T^\ast>M}]$ are lower than 
$\eta/4$. 
Just like above, we then use the universal approximation theorem to find a feed forward network function $T_{\hat \theta}$, which we can assume without loss of generality to be upper-bounded by $M$, such that   
\beq \label{eq1}
\EE_\PP|T^\ast-\FGen_{\hat \theta} | \leq \frac{\eta}{2}\quad \mbox{and} \quad 
\EE_\QQ|T^\ast-\FGen_{\hat \theta}|\unit_{T^\ast\leq M}\leq \frac{\eta}{4} e^{-M}
\eeq
We then write 
\beqa \label{eq2} 
\EE_\QQ[e^{T^\ast} - e^{\FGen_{\hat \theta}}] 
&=& \EE_\QQ[(e^{T^\ast} - e^{\FGen_{\hat \theta}})\unit_{T^\ast \leq M}]+ \EE_\QQ[(e^{T^\ast} - e^{\FGen_{\hat \theta}})\unit_{T^\ast>M}] \nonumber  \\
%&\leq& \EE_\QQ[|e^{T^\ast} - e^{\FGen_{\hat \theta}}| \unit_{T^\ast \leq M}] + \EE_\QQ[e^{T^\ast}\unit_{T^\ast>M}] \nonumber \\
&\leq& e^M \EE_\QQ[|T^\ast - \FGen_{\hat \theta}| \unit_{T^\ast \leq M}] + \EE_\QQ[e^{T^\ast}\unit_{T^\ast>M}]\nonumber \\
&\leq& \frac{\eta}{4} + \frac{\eta}{4} \\
&\leq& \frac{\eta}{2}
\eeqa
 where the inequality in the second line  arises from the convexity and positivity of $\exp$.  Eqns.~\ref{eq1} and \ref{eq2},  together with the triangular inequality,  lead to Eqn.~ \ref{proof-inequality}, which proves the Lemma. 
 
 \end{proof}

\begin{lemma}[Lemma \ref{lemma:estimation} restated]
Let $\eta >0$. Given a family $\mathcal{\cF}$ of neural network functions  $T_{\theta}$  with parameters $\theta$ in some compact domain $\Theta \subset \R^k$, there exists $N \in \mathbb{N}$ such that 
\beq \label{estim} 
 \forall  n\geq N, \quad \mathrm{Pr}\left(| \widehat{I(X;Z)}_n - I_\cF(X,Z) | \leq \eta \right) = 1 
 \eeq
\end{lemma}

 \begin{proof} 
We start by using the triangular inequality to write,
\beq \label{Equ:triang_ineq} 
| \widehat{I(X;Z)}_n - \sup_{T_\theta \in \cF} \hat{I}(T_\theta)  | \leq \sup_{\FGen_\theta \in \cF} |\EE_\PP [\FGen_{ \theta}] - \EE_{\PP_n}[ \FGen_{ \theta}]| + 
\sup_{\FGen_\theta \in \cF} |\log \EE_\QQ[e^{\FGen_{ \theta}}]  - \log \EE_{\QQ_n} [e^{\FGen_{ \theta}}] |
\eeq
The continuous function $(\theta, \omega) \to T_\theta(\omega)$, defined on the compact domain $\Theta \times \Omega $, is bounded. So the functions $T_\theta$ are uniformly bounded by a constant $M$, i.e $ |T_\theta| \leq M$ for all $\theta\in \Theta$. 
Since $\log$ is Lipschitz continuous with constant $e^{M}$ in the interval $[e^{-M}, e^{M}]$,  we have
\beq \label{Equ:Lipschitz} 
|\log \EE_\QQ[e^{\FGen_{ \theta}}]  - \log \EE_{\QQ_n} [e^{\FGen_{ \theta}}]  |  \leq e^{M}  | \EE_\QQ[e^{\FGen_{ \theta}}] - \EE_{\QQ_n} [e^{\FGen_{ \theta}}] | 
\eeq
Since $\Theta$ is compact and the feedforward network functions are continuous, the families of functions $T_\theta$ and $e^{T_\theta}$  satisfy the uniform law of large numbers~\citep{deGeer2006Mestimators}. Given $\eta >0$ we can thus choose $N \in \NN$ such that $\forall  n\geq N$ and with probability one, 
\beq
\sup_{T_\theta \in \cF} |\EE_\PP [\FGen_{ \theta}] - \EE_{\PP_n}[\FGen_{ \theta}] \leq \frac{\eta}{2} \quad \mbox{and} \quad  \sup_{T_\theta \in \cF}| \EE_\QQ[e^{\FGen_{ \theta}}] - \EE_{\QQ_n} [e^{\FGen_{ \theta}}] |  \leq \frac{\eta}{2} e^{-M} 
\eeq 
Together with Eqns.~\ref{Equ:triang_ineq}  and \ref{Equ:Lipschitz}, this leads to
\beq
| \widehat{I(X;Z)}_n - \sup_{T_\theta \in \cF} \hat{I}(T_\theta)  | \leq \frac{\eta}{2} + \frac{\eta}{2} = \eta
\eeq

\end{proof} 

 \begin{theo}[Theorem \ref{theo:consistency} restated]
MINE is strongly consistent. %, in the sense of Def~\ref{def-consistency}.  
 \end{theo} 
 
\begin{proof} 
Let $ \epsilon > 0$. We apply the two Lemmas to find a a family of neural network function $\mathcal{F}$ and  $N \in \NN$ such that (\ref{approx}) and  (\ref{estim}) hold with $\eta = \epsilon/2$. 
By the triangular inequality, for all $n\geq N$ and with probability one, we have: 
\beq
| I(X, Z)-\widehat{I(X;Z)}_n | \quad \leq \quad |I(X, Z) - \sup_{T_\theta \in \cF} \hat{I}(T_\theta)| + |\widehat{I(X;Z)}_n - I_\cF(X,Z) | \leq \quad \epsilon 
\eeq
which proves consistency. 
 \end{proof}

\subsubsection{Sample complexity proof}

\begin{theo}[Theorem \ref{theo:rate} restated]
Assume that the functions $T_{\theta}$ in $\cF$ are $L$-Lipschitz with respect to the parameters $\theta$; and that both $T_{\theta}$ and $e^{T_\theta}$ are $M$-bounded (i.e., $|T_\theta|, e^{T_\theta} \leq M$). The domain $\Theta \subset \R^d$ is bounded, so that $\|\theta\| \leq K$ for some constant $K$.  Given any values $\epsilon,\delta$  of the desired accuracy and confidence parameters, we have,  \beq  
\mathrm{Pr}\left(| \widehat{I(X;Z)}_n - I_\cF(X,Z) | \leq \epsilon \right) \geq 1- \delta
\eeq
whenever the number $n$ of samples satisfies 
\beq \label{num_samples}
n \geq \frac{ 2M^2 (d\log (16K L \sqrt{d} /\epsilon) + 2d M + \log(2/\delta))}{\epsilon^2}
\eeq 
\end{theo} 

\begin{proof}
%The proof is adapted from Theorem 3.1 of \citet{Arora2017}. 
The assumptions of Lemma \ref{lemma:estimation} apply, so let us begin with Eqns.~\ref{Equ:triang_ineq}  and \ref{Equ:Lipschitz}. By the Hoeffding inequality, for all function $f$, 
\beq \label{Eq:hoeffding}
\mathrm{Pr}\left(|\EE_\QQ[f] - \EE_{\QQ_n} [f]| \ > \frac{\epsilon}{4} \right) \leq 2 \exp{(-\frac{\epsilon^2 n}{2M^2})} 
\eeq
To extend this inequality to a uniform inequality over {\it all} functions  $T_\theta$ and $e^{T_\theta}$,   the standard technique  is to choose a minimal cover of the domain $\Theta\subset \R^d$ by a finite set of small balls of radius $\eta$, $\Theta \subset  \cup_{j} B_\eta(\theta_j)$, and to use the union bound.  The minimal cardinality of such covering is bounded by the {\it covering number} $N_\eta(\Theta)$ of  $\Theta$, known to satisfy\citep{Shalev-Schwartz2014}
\beq 
N_\eta(\Theta) \leq \left( \frac{2K \sqrt{d}}{\eta}\right)^d 
\eeq 
Successively applying a union bound in Eqn~\ref{Eq:hoeffding} with the set of functions $\{\FGen_{\theta_j}\}_j$ and 
$\{e^{\FGen_{\theta_j}}\}_j$ gives
\beq \label{union_bound}
\mathrm{Pr}\left(\max_j |\EE_\QQ[\FGen_{\theta_j}] - \EE_{\QQ_n} [\FGen_{\theta_j}]| > \frac{\epsilon}{4} \right) \leq 2 N_\eta(\Theta) \exp{(-\frac{\epsilon^2 n}{2M^2})} 
\eeq
and 
\beq \label{union_bound_exp}
\mathrm{Pr}\left(\max_j |\EE_\QQ[e^{\FGen_{\theta_j}}] - \EE_{\QQ_n} [e^{\FGen_{\theta_j}}]| > \frac{\epsilon}{4} \right) \leq 2 N_\eta(\Theta) \exp{(-\frac{\epsilon^2 n}{2M^2})} 
\eeq
We now choose the ball radius to be $\eta = \frac{\epsilon}{8L} e^{-2M}$. Solving for $n$ the inequation, 
\beq 
2 N_\eta(\Theta) \exp{(-\frac{\epsilon^2 n}{2M^2})}  \leq  \delta
\eeq
we deduce from Eqn~\ref{union_bound} that, whenever Eqn \ref{num_samples} holds, with probability at least $1-\delta$, for all $\theta \in \Theta$,  
\begin{align} 
|\EE_\QQ[\FGen_{\theta}] - \EE_{\QQ_n}[\FGen_{ \theta}]| &\leq |\EE_\QQ[\FGen_{\theta}] - \EE_{\QQ}[\FGen_{ \theta_j}]| + |\EE_\QQ[\FGen_{\theta_j}] - \EE_{\QQ_n}[\FGen_{ \theta_j}]| + |\EE_{\QQ_n}[\FGen_{\theta_j}] - \EE_{\QQ_n}[\FGen_{ \theta}]| \nonumber \\
& \leq \frac{\epsilon}{8} e^{-2M} + \frac{\epsilon}{4}  + \frac{\epsilon}{8} e^{-2M} \nonumber \\
& \leq \frac{\epsilon}{2}
\end{align} 
Similarly, using Eqn \ref{Equ:Lipschitz} and \ref{union_bound_exp}, we obtain that with probability at least $1-\delta$, 
\beq 
|\log \EE_\QQ[e^{\FGen_{\theta}}] - \log \EE_{\QQ_n}[e^{\FGen_{ \theta}}]| \leq \frac{\epsilon}{2} 
\eeq 
and hence using the triangular inequality,
\beq 
| \widehat{I(X;Z)}_n - I_\cF(X,Z) | \leq \epsilon
\eeq
\end{proof}

%\subsubsection{Gradient bias correction}
%\label{ssec:debiasing_proof}

%\begin{theo}[Bias of gradient estimat] \label{improved_bias}

%\end{theo}
\subsubsection{Bound on the reconstruction error}
  \label{ssec:recon_bound}
  Here we clarify relationship between reconstruction error and mutual information, by proving the bound in Equ~\ref{equ:recons_bound}. We begin with a definition:

\begin{definition}[Reconstruction Error] 
We consider encoder and decoder models giving conditional distributions $q(z|x)$ and $p(x|z)$ over the data and latent variables. If $q(x)$ denotes the marginal data distribution, the reconstruction error is defined as
  \begin{align}
  \mathcal{R} = \EE_{\bm{x}\sim q(\bm{x})} \EE_{\bm{z} \sim q(\bm{z}|\bm{x})}[-\log p(\bm{x}|\bm{z}) ]
  \end{align}
\end{definition}

We can rewrite the reconstruction error in terms of the joints $q(\bm{x},\bm{z}) = q(\bm{z}|\bm{x}) p(\bm{x})$ and $p(\bm{x},\bm{z}) = p(\bm{x}|\bm{z}) p(\bm{z})$.  Elementary manipulations give:
  \beq
%  \mathcal{R} &=& \EE_{(x,z)\sim p_{\tiny \mbox{dec}}} \log\frac{p_{\tiny \mbox{dec}}(x,z)p(z)}{p_{\tiny %\mbox{dec}}(x,z)p_{\tiny \mbox{enc}}(x,z)}
% \nonumber\\
    \mathcal{R} =\EE_{(\bm{x},\bm{z})\sim q(\bm{x}, \bm{z})} \log\frac{q(\bm{x},\bm{z})}{p(\bm{x},\bm{z})} - \EE_{(\bm{x},\bm{z})\sim q(\bm{x}, \bm{z})} \log q(\bm{x},\bm{z}) +  \EE_{\bm{z}\sim q(\bm{z})} \log p(\bm{z})  
    \eeq
 where $q(\bm{z})$ is the aggregated posterior.  The first term is the  KL-divergence $\KL{q}{p}$ ; the second term is the joint entropy $H_{q}(\bm{x},\bm{z})$. The third term can be written as 
 \[
 \EE_{\bm{z}\sim q(\bm{z})} \log p(\bm{z})   = - \KL{q(\bm{z})}{p(\bm{z})} - H_{q}(\bm{z}) 
 \]
%where   where the inequality follows from the positivity of the KL-divergence. inally by using the identity 
 % &=& \KL{p_{\tiny \mbox{dec}}}{p_{\tiny \mbox{enc}}} + H_{p_{\tiny \mbox{dec}}}(x,z) \nonumber \\
     %&=& \KL{p_{\tiny \mbox{dec}}}{p_{\tiny \mbox{enc}}} - H_{p_{\tiny \mbox{dec}}}(x|z)
   %\eeqa
Finally, the identity
\beq
H_{q}(\bm{x},\bm{z}) - H_{q}(\bm{z}) := H_{q}(\bm{z}|\bm{x}) = H_{q}(\bm{z}) - I_{q}(\bm{x},\bm{z})
\eeq
yields the following expression for the reconstruction error: 
\beq
    \mathcal{R}  = \KL{q(\bm{x},\bm{z})}{p(\bm{x},\bm{z})} - \KL{q(\bm{z})}{p(\bm{z})} - I_{q}(\bm{x},\bm{z}) + H_{q}(\bm{z})
\eeq
Since the KL-divergence is positive, we obtain the bound: 
  \begin{align}
    \mathcal{R} \leq \KL{q(\bm{x}, \bm{z})}{p(\bm{x}, \bm{z})} - I_{q}(\bm{x},\bm{z}) + H_{q}(\bm{z})
  \end{align}
which is tight whenever the induced marginal $q(\bm{z})$ matches the prior distribution $p(\bm{z})$. 

\subsection{Embeddings for bi-direction 25 Gaussians experiments}
Here (\Fig~\ref{fig:bi_directional2}) we present the embeddings for the experiments corresponding to \Fig~\ref{fig:bi_directional}.
\begin{figure*}[t]
\title{Embeddings for bi-directional 25 Gaussians experiments}
\begin{minipage}{0.2\textwidth}
\centering
(a) ALI
\includegraphics[width=0.99\columnwidth, clip=True, trim=40 20 40 40]{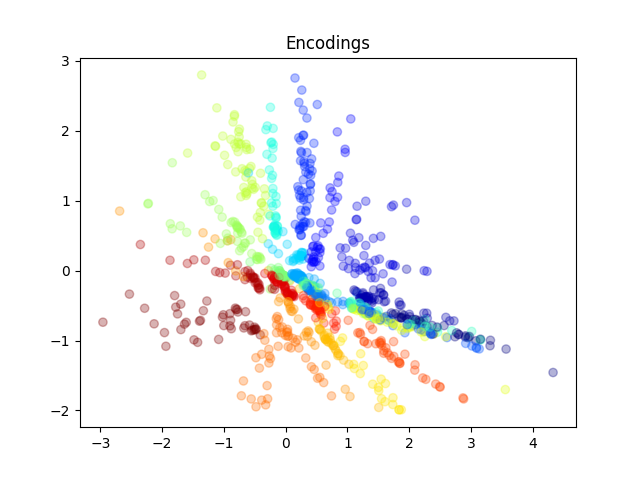}
\end{minipage}
\hfill
\begin{minipage}{0.2\textwidth}
\centering
(b) ALICE (L2)
\includegraphics[width=0.99\columnwidth, clip=True, trim=40 20 40 40]{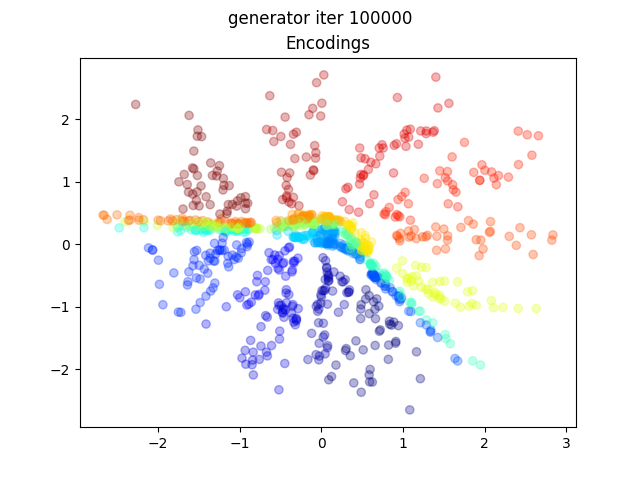}
\end{minipage}
\hfill
\begin{minipage}{0.2\textwidth}
\centering
(c) ALICE (A)
\includegraphics[width=.99\linewidth, clip=True, trim=40 20 40 40]{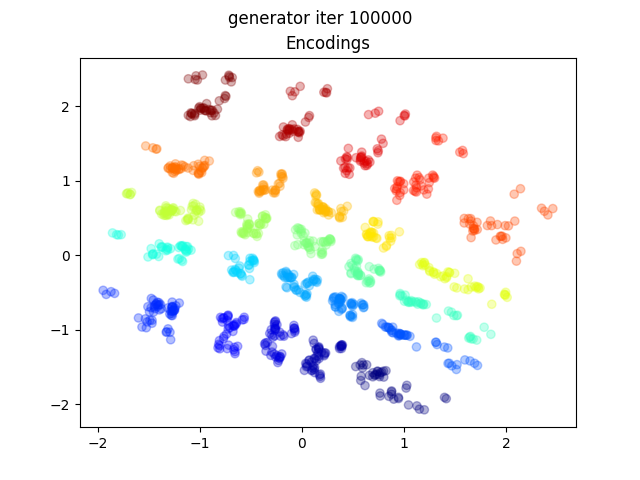}
\end{minipage}
\hfill
\begin{minipage}{0.2\textwidth}
\centering
(d) MINE
\includegraphics[width=.99\linewidth, clip=True, trim=40 20 40 40]{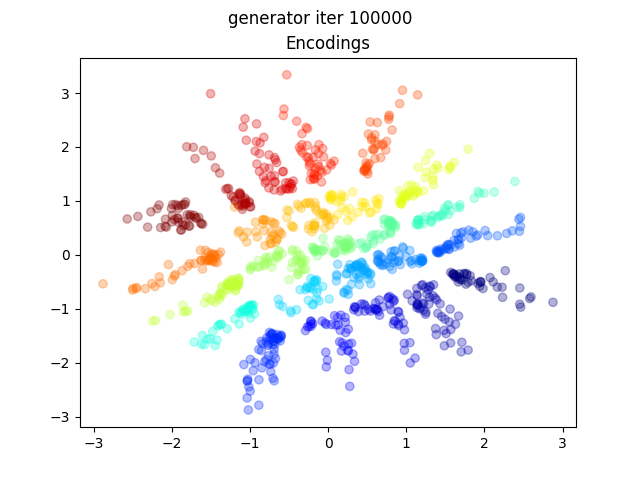}
\end{minipage}
\caption{Embeddings from adversarially learned inference (ALI) and variations intended to increase the mutual information.
Shown left to right are the baseline (ALI), ALICE with the L2 loss to minimize the reconstruction error, ALI with an additional adversarial loss, and MINE.}
\label{fig:bi_directional2}
\end{figure*}

\end{document}